\documentclass{article}

\usepackage{microtype}
\usepackage{graphicx}
\usepackage{subfigure}
\usepackage{booktabs} 

\usepackage{amssymb}
\usepackage{bbm}
\usepackage{amsmath}
\usepackage[noend]{algorithmic}
\usepackage{algorithm}
\usepackage{color}

\usepackage{hyperref}

\usepackage[utf8]{inputenc}
\usepackage{amsmath}
\usepackage{amsfonts}
\usepackage{amssymb}
\usepackage{bbm}
\usepackage{mathtools}
\usepackage{float}

\usepackage{calc}
\usepackage{parskip}
\usepackage{xcolor}
\usepackage{hyperref}
\usepackage{epsfig}
\usepackage{subfigure}
\usepackage{verbatim}
\usepackage{scalerel}
\usepackage{rotating}
\usepackage{enumitem}
\usepackage{wrapfig}
\usepackage{cleveref}
\usepackage{nicefrac}
\usepackage{bbm}
\usepackage{array}
\usepackage{soul}
\usepackage{float}

\newcommand{\timebr}{{\fontfamily{lmss}\selectfont{APR}}}
\newcommand{\timebrs}{{\timebr} }

\newcommand{\scsh}{{\fontfamily{lmss}\selectfont{SSH}}}
\newcommand{\scshs}{{\scsh} }

\newcommand{\insertfigcaptionspacing}{\vspace{-0.1in}}

\renewcommand\algorithmiccomment[1]{%
  \textcolor{darkblue}{\scriptsize \incmtt{\hfill\#\,{#1}}}%
}

\newcommand{\insertprethmspacing}{\vspace{0.1in}}

\newcommand{\muii}[1]{\mu_{#1}}
\newcommand{\muone}{\muii{1}}
\newcommand{\mutwo}{\muii{2}}
\newcommand{\muthree}{\muii{3}}
\newcommand{\mui}{\muii{i}}
\newcommand{\muj}{\muii{j}}
\newcommand{\muk}{\muii{k}}
\newcommand{\mun}{\muii{n}}

\newcommand{\mupii}[1]{\mu_{(#1)}}
\newcommand{\mupone}{\mupii{1}}
\newcommand{\muptwo}{\mupii{2}}
\newcommand{\mupthree}{\mupii{3}}

\newcommand{\mupk}{\mupii{k}}
\newcommand{\mupn}{\mupii{n}}

\newcommand{\thetaii}[1]{\theta_{#1}}
\newcommand{\thetaone}{\thetaii{1}}

\newcommand{\thetaj}{\thetaii{j}}
\newcommand{\thetak}{\thetaii{k}}

\newcommand{\thetap}{{\theta'}}
\newcommand{\thetapii}[1]{\theta_{(#1)}}

\newcommand{\thetapj}{\thetapii{j}}
\newcommand{\thetapk}{\thetapii{k}}

\newcommand{\thetappii}[1]{\theta'_{(#1)}}

\newcommand{\thetappj}{\thetappii{j}}
\newcommand{\thetappk}{\thetappii{k}}

\newcommand{\fk}{f_k}
\newcommand{\fpk}{f'_k}

\newcommand{\thetabrii}[1]{\theta_{(#1)}}

\newcommand{\thetabri}{\thetabrii{i}}
\newcommand{\thetabrj}{\thetabrii{j}}

\newcommand{\Niirr}[2]{N_{#1}(#2)}
\newcommand{\Nir}{\Niirr{i}{r}}
\newcommand{\muhatiinn}[2]{\widehat{\mu}_{#1}(#2)}
\newcommand{\muhatiNir}{\muhatiinn{i}{r}}

\newcommand{\speedfunc}{\lambda}
\newcommand{\scalefunc}{\lambda}

\newcommand{\scalefuncinv}{\lambda^{-1}}

\newcommand{\Srr}[1]{S_{#1}}
\newcommand{\Sr}{\Srr{r}}

\newcommand{\Tstar}{T^\star}

\newcommand{\kstar}{k^\star}

\newcommand{\Deltaii}[1]{\Delta_{#1}}
\newcommand{\Deltai}{\Deltaii{i}}

\newcommand{\Deltaone}{\Deltaii{1}}
\newcommand{\Deltatwo}{\Deltaii{2}}
\newcommand{\Deltathree}{\Deltaii{3}}
\newcommand{\Deltapii}[1]{\Delta_{(#1)}}

\newcommand{\Deltapk}{\Deltapii{k}}

\newcommand{\Acaldelsf}{\Acal_{\delta,\scalefunc}}
\newcommand{\Acaltildedelsf}{\widetilde{\Acal}_{\delta,\scalefunc}}
\newcommand{\Atilde}{\widetilde{A}}

\newcommand{\filrr}[1]{\Fcal_{#1}}
\newcommand{\filr}{\filrr{r}}

\newcommand{\filrho}{\filrr{\rho}}
\newcommand{\filprr}[1]{\Fcal'_{#1}}

\newcommand{\filpN}{\filprr{N}}

\newcommand{\actrr}[1]{a_{#1}}
\newcommand{\actr}{\actrr{r}}

\newcommand{\timrr}[1]{t_{#1}}
\newcommand{\timr}{\timrr{r}}

\newcommand{\fdbrr}[1]{O_{#1}}
\newcommand{\fdbr}{\fdbrr{r}}

\newcommand{\Nkkrr}[2]{N_{(#1)}(#2)}
\newcommand{\Nkr}{\Nkkrr{k}{r}}
\newcommand{\Njr}{\Nkkrr{j}{r}}
\newcommand{\Nkrho}{\Nkkrr{k}{\rho}}

\newcommand{\Ntildekkrr}[2]{\widetilde{N}_{(#1)}(#2)}
\newcommand{\Ntildekr}{\Ntildekkrr{k}{r}}

\newcommand{\Ntildekrho}{\Ntildekkrr{k}{\rho}}

\newcommand{\Nrr}[1]{N({#1})}
\newcommand{\Nr}{\Nrr{r}}
\newcommand{\Ykkss}[2]{Y_{#1,#2}}
\newcommand{\Yks}{\Ykkss{k}{s}}

\newcommand{\Ldag}{L^\dagger}
\newcommand{\kptilde}{k_{\tilde{p}}}
\newcommand{\kptildepo}{k_{\tilde{p}+1}}
\newcommand{\kptildemo}{k_{\tilde{p}-1}}

\newcommand{\insertAlgoTimeBatchRacing}{
\setlength{\textfloatsep}{15pt}
\begin{algorithm}[t]
\caption{Adaptive Parallel Racing (\timebr)}\label{alg:fixed_confidence}
\begin{algorithmic}[1]
\STATE {\bfseries Input:} confidence $1-\delta$, (constant) parameter $\beta>1$.
\STATE $r \leftarrow 1$,  
 $N_i(0) \leftarrow 0, \forall i\in[n]$
 \STATE $A_1\leftarrow \emptyset$, $\;S_1 \leftarrow [n]$, $t_1 = \lambda(n)$, $q_1 = 1$
\WHILE{$|A_r| = 0$}
    \STATE Pull each arm in $S_r$ $q_r$ times, taking $\lambda(|S_r|q_r)$ time.
    \STATE $N_i(r) = N_i(r-1) + q_r\mathbbm{1}\{i\in S_r\}\ \forall i\in[n]$
    \STATE $\hat{\mu}_{i,r} \leftarrow$ empirical mean of arm $i$ at round $r,\ \forall i\in[n]$
    \STATE $A_{r+1}\leftarrow$ \COMMENT{See~\eqref{eqn:confintervals}}
    \label{lin:surviving_arms_A}
    \newline
     \vphantom{T}\hspace{0.27in}
            $\{i\in S_r | L_i(r,\delta) > \underset{j\in S_r\setminus\{i\}}{\max} U_j(r,\delta)\}$
    \STATE $S_{r+1}\leftarrow$ \newline
     \vphantom{T}\hspace{0.27in}
            $\{i\in S_r | U_i(r,\delta) > \underset{j\in S_r}{\max}\ L_j(r,\delta)\}\setminus
            A_{r+1}$\label{lin:surviving_arms_S}
    \STATE $q_{r+1} \leftarrow \left\lfloor \lambda^{-1}(\beta^{r}t_{1})/|S_{r+1}| \right\rfloor$
    \STATE $r\leftarrow r + 1$.
\ENDWHILE
\STATE {\bfseries Return} $A_r$
\end{algorithmic}
\end{algorithm}
}

\newcommand{\insertAlgoStageCombSeqHalving}{
\begin{algorithm}[t]
\caption{\;Staged Sequential Halving (\scsh)\label{alg:fixed_budget}}
\begin{algorithmic}[1]
\STATE {\bfseries Input:} time budget $T$, number of stages to combine $k$
\STATE $m = \lambda^{-1}\left(T/\left\lceil\log_{2^k}(n)\right\rceil\right)$ \COMMENT{pulls per stage}
\STATE $r_f \leftarrow \lceil \log_{2^{k}}(n)\rceil$, $S_0 \leftarrow [n]$
\FOR{$r \in \{0,\ldots, r_f-1\}$}
    \STATE \parbox[t]{\dimexpr\textwidth-\leftmargin-\labelsep-\labelwidth}{%
    Sample each arm $i \in S_r$, $t_r = \left\lfloor m/|S_r|\right\rfloor$ times.\strut}
    \STATE \parbox[t]{\dimexpr\textwidth-\leftmargin-\labelsep-\labelwidth}{%
    Let $S_{r+1}$ be the set of $\left\lceil |S_r|/2^{k} \right\rceil$ arms in $S_r$ with
the \\ highest empirical mean \strut}
\ENDFOR
\STATE {\bfseries Return} The singleton element in $S_{r_f}$.
\end{algorithmic}
\end{algorithm}
}

\newcommand{\insertFigSettingIllus}{%
\begin{figure}[t]
\includegraphics[width=3.0in]{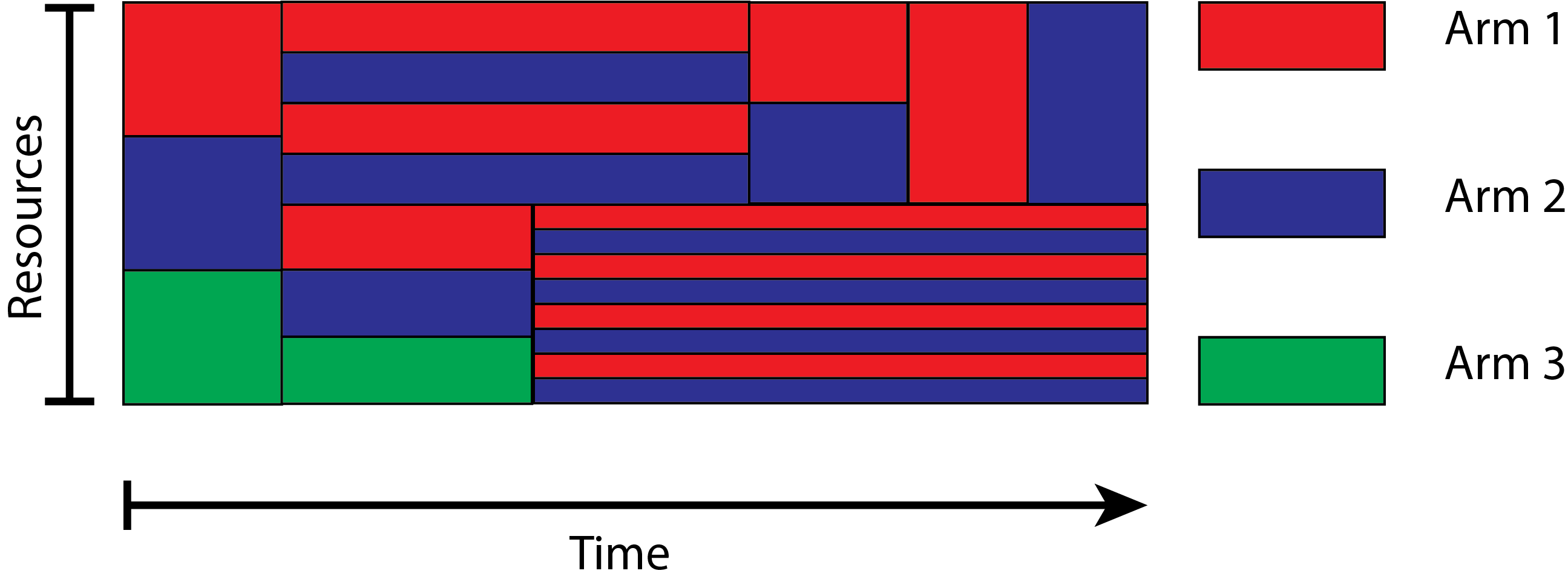}
\vspace{-0.15in}
\caption{\small
An example of the setting using $n=3$ arms.
An algorithm has a \emph{divisible resource} (y-axis) that can be
split up to generate samples by pulling the 3 arms. Each pull takes \emph{time} (x-axis) to complete.
Allocating more resources to a pull produces faster results,
but at the cost of decreased throughput,
e.g. using twice the resources for a pull will provide results only 1.5x faster.
When a pull completes, the resources used
are freed and can be reallocated to new pulls.
\label{fig:settingillus}\vspace{-0.10in}
}
\end{figure}
}

\newcommand{\insertFigDistParIllus}{%
\begin{figure*}
\centering     
\includegraphics[width=6.0in]{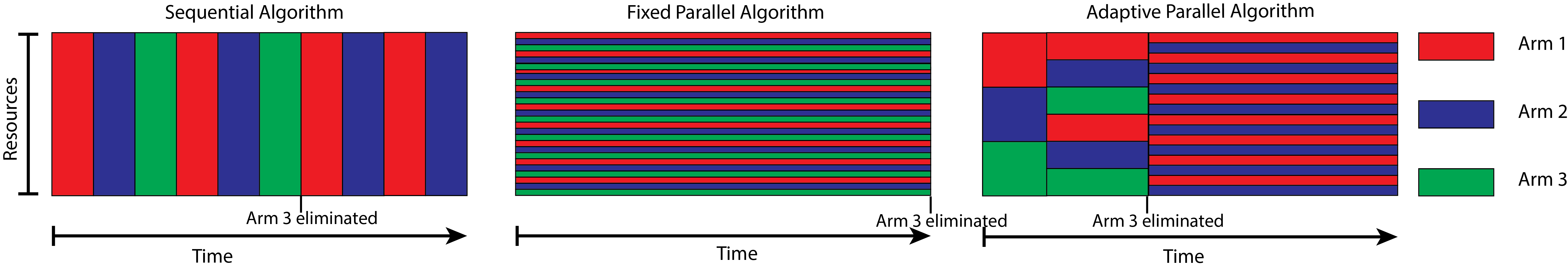}
\vspace{-0.1in}
\caption{\small{\emph{Left:} A sequential strategy will allocate all
resources to pulls one at a time.
This 
enables the algorithm to obtain information about individual pulls sooner for replanning; after
pulling all arms twice, it is able to eliminate the green arm.
However, this reduces throughput due to sublinear scaling, which can be inefficient in the long run.
\emph{Middle:} A highly-parallel strategy with fixed batch size will have higher throughput.
However, since it does not have feedback, it may
evaluate the poor arms too many times and take longer to eliminate them.
\emph{Right:} Our algorithm, APR, for the fixed confidence setting,
adaptively
manages parallelism during execution based on the scaling function and task progress.
In this figure, the green arm is eliminated early, and then throughput is increased for the
remaining $2$ arms which may be harder to distinguish.
While the setting permits allocating the resources ``asynchronously'' (see
Fig.~\ref{fig:settingillus}), we show that an algorithm which operates synchronously, but
chooses the amount of parallelism adaptively is minimax optimal for this problem.
}
\label{fig:tradeoffillus}
}
\insertfigcaptionspacing
\end{figure*}
}

\newcommand{\insertFigFixedConf}{%
\begin{figure*}[h]
\includegraphics[width=\linewidth]{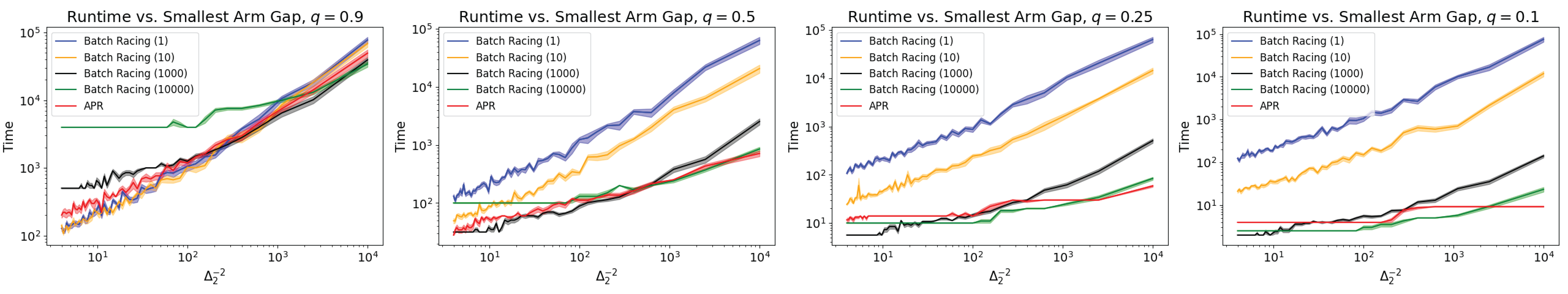}\label{fig:fixed_conf_experiments}
\vspace{-0.30in}
\caption{\small{\textbf{Fixed confidence synthetic results:} We vary the first gap $\Delta_2$ from $0.01$ to $0.5$ for $q \in \{0.9, 0.5, 0.25, 0.1\}$. A higher $\Delta_2^{-2}$ (right side) implies the problem is harder to solve as more arm pulls will be needed to separate the confidence bounds of the best arm from the rest. We present the mean and standard error across $10$ runs for all experiments. We observe that Algorithm~\ref{alg:fixed_confidence} (APR) is able to trade off between throughput and information accumulation effectively in most cases without additional hyperparameters, consistently performing near the best algorithms for each problem. The baseline algorithms perform well when their level of parallelism is well-suited to the scaling and $\Delta_i$ values of the problem, but are inconsistent across problems.\vspace{-0.2in}
}}
\end{figure*}
}

\newcommand{\insertFigFixedDeadline}{%
\begin{figure*}[h]
\centering     
\subfigure{\label{fig:c1h}\includegraphics[width=40mm]{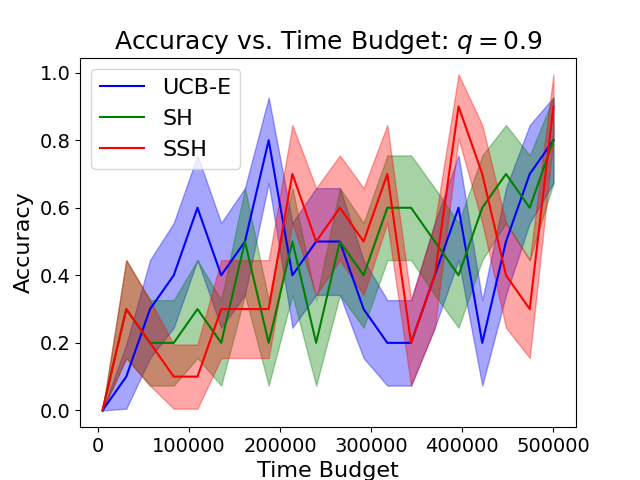}}
\subfigure{\label{fig:c1f}\includegraphics[width=40mm]{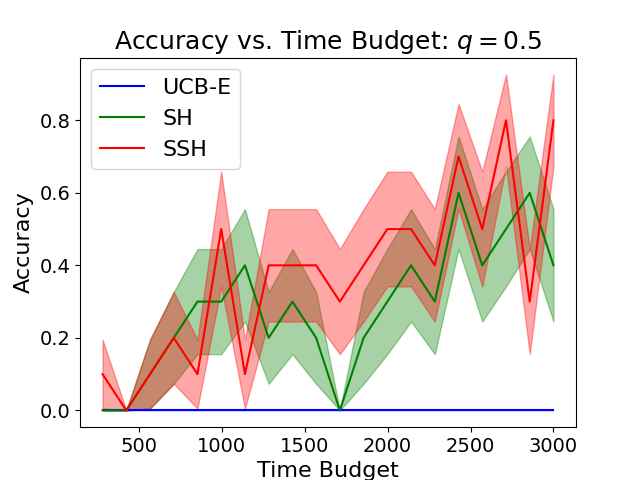}}
\subfigure{\label{fig:c1d}\includegraphics[width=40mm]{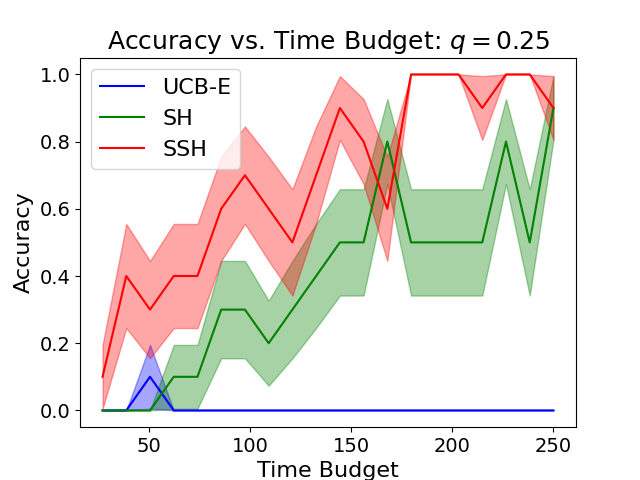}}
\subfigure{\label{fig:c1b}\includegraphics[width=40mm]{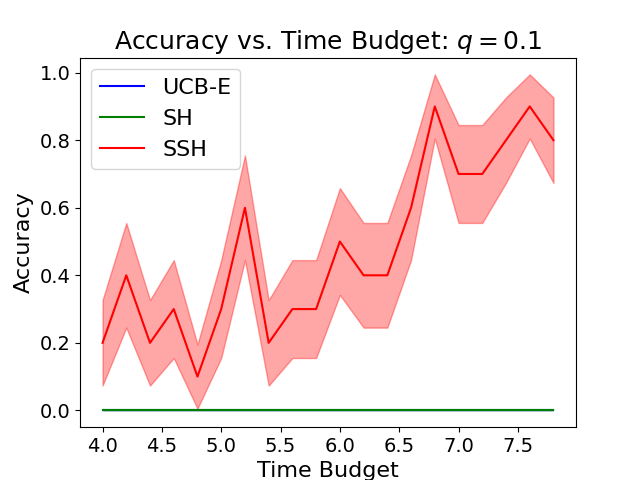}}
\vspace{-0.2in}

\subfigure{\label{fig:c1g}\includegraphics[width=40mm]{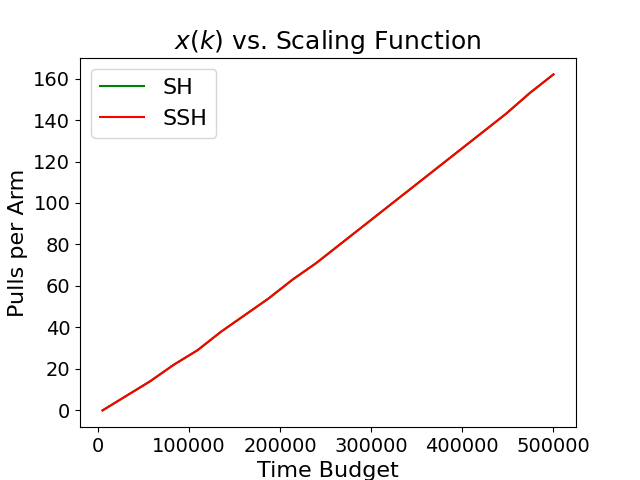}}
\subfigure{\label{fig:c1e}\includegraphics[width=40mm]{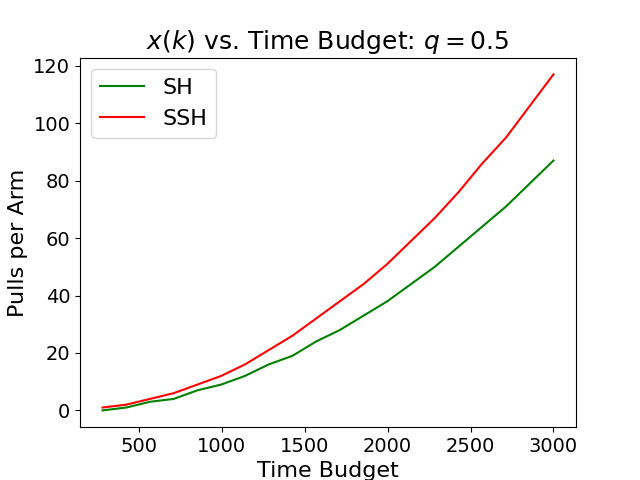}}
\subfigure{\label{fig:c1c}\includegraphics[width=40mm]{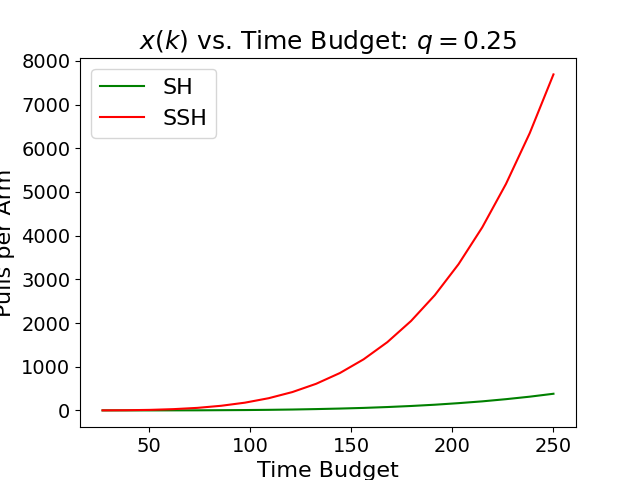}}
\subfigure{\label{fig:c1a}\includegraphics[width=40mm]{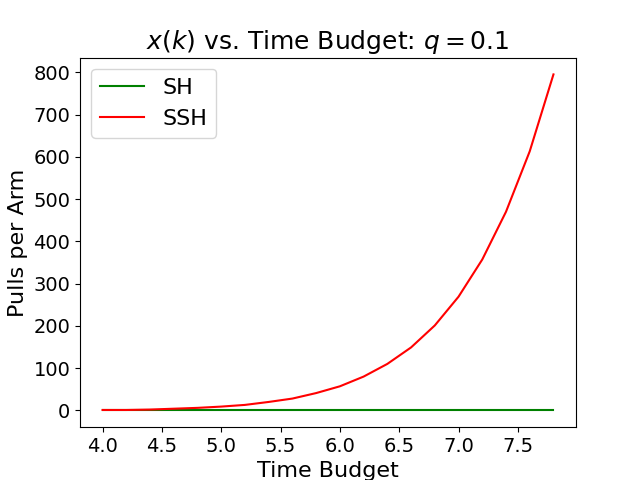}}
\vspace{-0.1in}
\caption{\small{\textbf{Fixed deadline synthetic results:} We vary the time budget for $q \in \{0.1,
0.25, 0.5, 0.9\}$, running for $10$ runs each and presenting the mean and standard error. We find
that SSH (red) consistently outperforms SH (green) and UCB-E (blue).
This difference is large when the scaling is poor ($q$ is small) as
SH does not pull each arm a sufficient number of times before each stage ends, 
and the cost of allocating all resources to individual pulls prevents UCB-E from
making progress in comparable time.
In contrast, \scsh{} does well because it prioritizes high throughput when the scaling is poor.
When $q=0.1$, SH and UCB-E perform poorly whereas when $q\geq 0.5$, \scsh{} has similar throughput
as SH, resulting in similar performance.}\label{fig:fixed_budget_exps}}
\end{figure*}
}

\usepackage{kk}      

\usepackage[accepted]{icml2021}

\icmltitlerunning{Resource Allocation in Multi-armed Bandit Exploration}

\begin{document}

\twocolumn[
\icmltitle{Resource Allocation in Multi-armed Bandit Exploration:
\\ Overcoming Sublinear Scaling with Adaptive Parallelism
}


\icmlsetsymbol{equal}{*}

\begin{icmlauthorlist}
\icmlauthor{Brijen Thananjeyan}{cal}
\icmlauthor{Kirthevasan Kandasamy}{cal}
\icmlauthor{Ion Stoica}{cal}
\icmlauthor{Michael I. Jordan}{cal}
\icmlauthor{Ken Goldberg}{cal}
\icmlauthor{Joseph E. Gonzalez}{cal}
\end{icmlauthorlist}

\icmlaffiliation{cal}{University of California, Berkeley}

\icmlcorrespondingauthor{Brijen Thananjeyan}{bthananjeyan@berkeley.edu}

\icmlkeywords{Machine Learning, ICML}

\vskip 0.3in
]



\printAffiliationsAndNotice{}  

\begin{abstract}
We study exploration in stochastic multi-armed bandits when we have access to a divisible resource
that can be allocated in varying amounts to arm pulls. We focus in particular on the allocation of
distributed computing resources, where we may obtain results faster by allocating more resources
per pull, but might have reduced throughput due to nonlinear scaling.
For example, in simulation-based scientific studies, an expensive simulation can be sped up by
running it on
multiple cores. This speed-up however, is partly offset by the communication among cores, which results in lower throughput than if fewer cores were allocated per trial to run more trials in parallel.
In this paper, we explore these trade-offs in two settings.
First, in a fixed confidence setting, we need to find the best arm with a given target
success probability as quickly as possible.
We propose an algorithm which trades off between information accumulation and throughput and show that the time taken can be upper bounded by the solution of a dynamic program whose inputs are the gaps between the sub-optimal and optimal arms. We also prove a matching hardness result. Second, we present an algorithm for a fixed deadline setting, where we are given a time deadline and need to maximize the probability of finding the best arm. We corroborate our theoretical insights with simulation experiments that show that the algorithms consistently match or outperform baseline algorithms on a variety of problem instances.
\end{abstract}

\section{Introduction}
\label{sec:intro}


In multi-armed bandit exploration, an agent draws samples from a set of $n$ arms, where,
upon pulling arm $i$, it receives a stochastic reward drawn from a
distribution with mean $\mui$.
The goal is to identify the best arm, i.e., $\argmax_i \mui$,
by adaptively choosing which arms to pull.
This problem is known in the literature as best-arm identification (BAI).
\insertFigSettingIllus

As an example, consider simulation-based studies in physics,
which are used for
estimating cosmological constants and controlling nuclear fusion
reactors~\citep{davis07supernovae,xing2019automating}.
Pulling an arm may correspond to running a stochastic simulation with specific values for
parameters
(cosmological constants or reactor parameters) which affect the output of the simulation.
Here,
searching for the optimal parameters is naturally modeled as a BAI problem.
BAI is also used in model selection~\citep{li2017hyperband},
A/B testing~\citep{howard2019sequential}, and other configuration tuning tasks.
Traditionally, BAI is studied in two settings: fixed confidence and fixed budget.
In the former, we must identify the best arm with a given target success probability,
while keeping the number of arm pulls to a minimum.
In the latter, we are given a budget of pulls and should maximize the probability
of identifying the best arm.

\insertFigDistParIllus

The focus of work in the BAI literature has been the sequential setting in which the agent can pull
only one arm at a time~\citep{bubeck2009pure,audibert2010best,even2002pac}.  There is also a line of work on the batch parallel setting, where the agent can pull a fixed number of arms at a
time~\citep{jun2016top}.
These formulations have their limitations; in particular, modern infrastructures for parallel and distributed computing enable us
to execute several arm pulls in parallel and control
the duration of arm pulls by assigning multiple resources to a single pull.
For example, physics simulations can exhibit strong elasticity in their resource
requirements, where the number of CPUs allocated to the same simulation can range from
a few tens to a few thousand CPUs.
Using more CPUs results in speedier outputs, but this speedup can be sublinear due to
communication and synchronization among the CPUs.

We develop algorithms for BAI where we are given access to a fixed
amount of a divisible resource to execute the arm pulls.
In addition to choosing which arms to pull, an algorithm must also determine how much of this
resource to allocate for each arm pull.
We have illustrated the setting in Fig.~\ref{fig:settingillus}.
While traditional formulations for BAI are stated in terms of the number of arm pulls,
in our setting it is meaningful to formulate this in terms of \emph{time constraints}.

By using more resources for a single arm pull, we obtain its result sooner;
however, there is diminishing value to allocating more resources to the same arm
pull, since typical distributed environments do not scale
linearly~\citep{venkataraman2016ernest,liaw2019hypersched}.
By allocating more resources for individual pulls and
obtaining their results sooner, we can use that information to
make refined decisions about which arms to try in subsequent iterations.
However, due to sublinear scaling, 
by \emph{parallelizing} arm pulls
(executing many arm pulls simultaneously with fewer resources each), 
we can complete 
more pulls per unit time and obtain more information about the arms.
This tradeoff between 
information accumulation (obtaining the results of a particular pull sooner)
and
throughput (number of arm pulls per unit of time)
is fundamentally different from the usual explore-exploit trade-off encountered in bandit problems.
In the latter, exploration is akin to testing several arms (and not arm pulls) and exploitation
is akin to testing fewer carefully selected arms, and
in our setting, an algorithm may choose to explore and/or exploit at different levels of
parallelism.
We have illustrated this trade-off in Figure~\ref{fig:tradeoffillus}.

In this work, we assume a \emph{known} scaling function
$\scalefunc$, where $\scalefunc(1/\alpha)$ is the time taken to complete a pull using a fraction
$\alpha$ of the resource.
To model the above trade-off, we will assume a diminishing returns property on $\scalefunc$.

\textbf{Our contributions} are as follows.
First, in the fixed confidence setting, we propose an algorithm, Adaptive Parallel Racing (\timebr), and bound its time complexity by
the solution to a dynamic
program (DP) whose inputs are the inverse squared gaps between the optimal arm and the other
arms.
We also prove a hardness result which shows that the expected time taken by any algorithm is 
lower bounded by this DP, demonstrating that this DP is fundamental to this problem.
Second, while our primary focus is in the fixed confidence setting,
 we also study a fixed deadline version of this problem, where we are given a time deadline, and
wish to maximize the probability of identifying the best arm.
We propose Staged Sequential Halving (\scsh), a simple variant of the popular sequential-halving strategy for BAI and bound its
failure probability.
Third, we corroborate these theoretical insights with empirical simulations on some synthetic
problems.
We observe that \timebrs performs as well as the best task-tuned baseline algorithms with fixed levels
of parallelism in the fixed confidence setting.
In the fixed deadline setting,
\scsh{} succeeds $20-90\%$ more often
than baselines which do not account for nonlinear scaling.

\vspace{-0.05in}
\subsection*{Related work}
\vspace{-0.05in}

Since multi-armed bandits were introduced by~\citet{thompson33sampling},
they have been widely studied as an abstraction
that formalizes the exploration-exploitation trade-offs that arise in decision-making
under uncertainty~\citep{robbins52seqDesign,auer03ucb}.
Best arm identification (BAI) is a special case of bandits which is generally studied in the sequential setting in which an agent adaptively evaluates one arm at a time in order to identify the best
arm~\citep{bubeck2009pure,gabillon2012best,
karnin2013almost,russo2016simple,bubeck2013multiple,kalyanakrishnan2010efficient,jamieson2014lil}.
However, this sequential setting can fall short of capturing the full range of trade-offs that arise when an agent may be able to evaluate several arms concurrently.

In recent work, \citet{jun2016top} formulate a parallel version of BAI in the fixed confidence setting using the confidence intervals from~\citet{jamieson2014lil}.  In their work, the level of parallelism remains fixed; i.e.,
the agent is allowed to select a batch of arms to pull at each round. In the current paper, we consider a setting that allows adaptive parallelism using a fixed resource, where we explicitly consider \emph{execution time} as a function of the batch size.
~\citet{grover2018best} study a similar setting where there is delayed feedback in BAI.
While their model allows handling parallel arm pulls, it does not allow adaptive resource allocation.
In addition,
a line of work has studied parallel bandits in the regret minimization
setting using Bayesian optimization~\citep{desautels2014parallelizing,kandasamy2018parallelised}. In this work, we study the BAI setting, which is a pure exploration problem and we do not assume a prior over the arms.

Our algorithm for the fixed confidence setting proceeds by constructing confidence intervals for
the arms' mean values, and then eliminates those arms which can be concluded to be non-optimal based
on these confidence intervals.
The construction of these confidence intervals is based on the law of the iterated
logarithm~\citep{jamieson2014lil}.
We also establish a hardness result in this setting which requires establishing lower bounds on the
sample complexity, and then translating this to a lower bound on the time.
For the sample complexity bounds, we use ideas from~\citet{kaufmann2016complexity}. 
In the fixed deadline setting, our algorithm is based on~\citet{karnin2013almost},
who proposed a sequential-halving (SH) strategy for sequential BAI.
SH splits the budget of arms into stages and eliminates the worst half in each stage. This
allows the algorithm to make more pulls of the promising arms.
We show that a naive extension of this strategy can perform poorly
 if the scaling function is poor and
propose an alternative algorithm that eliminates arms at a rate determined by the scaling function.

\section{Description of the Environment}
\label{sec:environment}

We begin by describing the bandit environment which will be used in both settings.
There are $n$ arms, denoted $[n] = \{1,\dots,n\}$.
When we pull arm $i \in [n]$, we observe a reward from a $1$-sub-Gaussian distribution
with expectation $\mui \in [0, 1]$.
The goal is to identify the best arm, i.e., we wish to find $\argmax_i \mui$.
We will assume that the best arm is unique, and, for ease of exposition, that the arms
are ordered in decreasing order.
Therefore, $\mu_1 > \mu_2 \geq \cdots \geq \mu_n$.
We define the \emph{gaps} $\Deltai$ as follows:
\begin{align*}
\vspace{-0.1in}
\numberthis \label{eqn:Deltai}
\hspace{-0.2in}
\Deltaone = \muone - \mutwo,
\hspace{0.4in}
\Deltai = \muone - \mui, \;\forall\,i\geq 2.
\vspace{-0.1in}
\end{align*}
%
Deviating from prior work on BAI,
we assume that we have access to a divisible resource which is to be used to execute the arm pulls.
The time taken to execute a single pull using a fraction $\alpha$ of this resource is given by
$\scalefunc(1/\alpha)$.
Here, $\scalefunc:\RR_+ \rightarrow\RR_+$ is an application-specific \emph{scaling function}
which is assumed to be known.
In particular, this implies that 
the time taken to execute $m\in\NN$ arm pulls by \emph{evenly} dividing the entire resource  is 
$\scalefunc(m)$; this interpretation will be useful in understanding our assumptions going forward.
It is reasonable to assume that scaling characteristics are known as they can either be modeled
analytically~\citep{zahedi2018amdahl},
or can be
profiled from historical experiments~\citep{venkataraman2016ernest,liaw2019hypersched}.
If $\scalefunc$ is unknown, we believe that there are significant limitations on what can be achieved in this setting.
Prior work in the operations research literature studying scheduling in distributed
systems also assumes that scaling characteristics are
known~\citep{berg2018towards,berg2020hesrpt}.

While $\scalefunc$ depends on the application,
we will make some assumptions to model practical use cases.
First, $\scalefunc$ is an increasing function with $\lambda(0) = 0$, which simply states that
executing more arm pulls requires more time and that zero pulls takes no time.
Second, ${\rm Range}(\scalefunc) = \RR_+$, 
which states that we cannot execute an unbounded number of pulls in a bounded amount of time;
hence, $\scalefuncinv:\RR_+\rightarrow\RR_+$.
Third, sublinear scaling---i.e., the diminishing returns of allocating
more resources to a single pull---can be captured via the following assumption.
For all $m_2\geq m_1 > 0$, $\delta_1, \delta_2 > 0$,
\begin{align}
\frac{\lambda\left(m_1+\delta_1\right) - \lambda\left(m_1\right)}{\delta_1}
\geq
\frac{\lambda\left(m_2+\delta_2\right) - \lambda\left(m_2\right)}{\delta_2}.
\label{eqn:speedfuncassumption}
\numberthis
\end{align}
That is, the change in the average time taken to do additional arm pulls is smaller
when there are more pulls already in the system.
This assumption is equivalent to concavity.
The above assumptions hold true for many choices for $\scalefunc$ in practice.
For instance, it is true for Amdahl's law and its variants that are popularly used to analytically
model speed-ups in multi-core and distributed
environments~\citep{amdahl1967validity,hill2008amdahl,zahedi2018amdahl}.
They have also been found to be empirically true in other application-specific
use cases~\citep{venkataraman2016ernest,liaw2019hypersched}.
We also wish to mention that such 
concavity assumptions are common in the operations research
literature when modeling diminishing returns~\citep{berg2018towards,berg2020hesrpt}.

Finally, as mentioned in Section~\ref{sec:intro},
we will allow an algorithm to allocate
its resources asynchronously,
where a fraction of the resource may be allocated for one set of pulls to be completed at a
certain throughput, and the remaining
fraction for another set of pulls at a different throughput.
Thus, if both these sets of pulls start at the same time,
they may finish at different times.
As we will see shortly, at least in the fixed confidence setting, a synchronous algorithm may be
sufficient as it matches a hardness result.
%


\section{Fixed Confidence Setting}
\label{sec:fixedconf}

In the fixed confidence setting, the decision-maker is given a target failure probability $\delta$,
and must find the best arm with probability of error at most $\delta$ while minimizing the time
required to do so.
A common approach for sequential fixed-confidence BAI maintains confidence intervals
for the mean values of the arms based on past observations;
when the upper confidence bound of an arm falls below the lower confidence bound of any other arm,
the algorithm eliminates that arm until eventually there is only one arm
left~\citep{jamieson2014lil,gabillon2012best}.

There are two major challenges in applying such confidence-interval-based algorithms in our setting.
They both arise due to nonlinear scaling and are a consequence of
the trade-off between information accumulation and throughput.
The first of these challenges is due to the fact that the arm mean values, or equivalently
the gaps, are unknown.
To illustrate this, consider an example with $n=2$ arms and let $\Delta = \muone-\mutwo$.
It is well known that differentiating between these two $1$-sub-Gaussian distributions requires
$N_\Delta\approx\Delta^{-2}$ samples from each arm.
For the purpose of this example\footnote{%
In this and the following example, we make simplifying assumptions in our treatment of $N_\Delta$.
For instance, we can usually only upper or lower bound $N_\Delta$ in terms of $\Delta^{-2}$.
Additionally, in adaptive
settings, it may be a random variable.
These simplifications are made to illustrate the challenges in our setup.
Our subsequent analysis will be rigorous.
\label{ftn:mqexample}
},
let us assume $N_\Delta=32$ and assume that the scaling function is
$\scalefunc(m) = m^{\nicefrac{1}{2}}$.
If we pull each arm one at a time, allocating all resources to each pull,
this will require $2\times N_\Delta = 64$ arm pulls
and hence take time $2\times N_\Delta \times \scalefunc(1) = 64$.
If instead all $64$ pulls are
executed simultaneously with $1/64$ of the resources for each pull, this takes time
$\scalefunc(64) = 8$, which is significantly less.
However, knowing the right amount of parallelism requires information about the $\Delta$ value
which is not available to the algorithm.
If we parallelize more than necessary, we will be executing more arm pulls than necessary which can
take more time.
For instance, pulling each arm $512$ times will take time $\scalefunc(1024) = 32$.
Therefore, the first challenge for an algorithm is to choose the right amount of parallelism
without knowledge of the gaps.

The second challenge arises from the fact that the ordering of the arms is unknown.
As an example, consider a problem with three arms with $\muone > \mutwo > \muthree$ and
$\Deltaone = \Deltatwo=\muone-\mutwo$, and $\Deltathree = \muone-\muthree$.
Let $N_{\Deltatwo} \,\big(\approx \Deltatwo^{-2}\big)$,
$N_{\Deltathree} \,\big(\approx \Deltathree^{-2}\big)$, and
$N_{\Deltaone} = \max(N_{\Deltatwo}, N_{\Deltathree}) = N_{\Deltatwo}$
denote the number of pulls required from the second, third, and first arms, respectively. For simplicity, let us assume that the algorithm is aware of the
$N_{\Deltai}$ values, but does not know the ordering; i.e., which arms are the first, second, and
third.
Let $\scalefunc(m) = m^{\nicefrac{1}{2}}$.
First, let $N_{\Deltaone} = N_{\Deltatwo} = 300$ and let $N_{\Deltathree} = 5$.
As $N_{\Deltathree}$ is small, it is efficient to eliminate the third arm first and then
differentiate between the top two arms.
However, since the permutation of the arms is unknown, the algorithm will first pull each
arm $N_{\Deltathree}$ times, eliminate the third arm, and then
pull the remaining two arms $N_{\Deltatwo} - N_{\Deltathree}$ times each.
This takes time $\scalefunc(3N_{\Deltathree}) + \scalefunc(2(N_{\Deltatwo} - N_{\Deltathree}))
\approx 28.16$.
Alternatively, consider a different example where $N_{\Deltathree} = 100$; i.e.,
the third arm is harder to distinguish.
Here, first eliminating the third arm and then proceeding to the remaining arms takes
time $\scalefunc(3N_{\Deltathree}) + \scalefunc(N_{\Deltatwo} - N_{\Deltathree}) \approx 37.32$.
However, simply pulling all three arms $3N_{\Deltatwo}$ times, thus eliminating both the second
and the third arm simultaneously, takes time $\scalefunc(3N_{\Deltatwo})=30$ which is faster.
This example illustrates that due to nonlinear scaling, it might be better to pull even the
sub-optimal arms a larger number of times as the improved throughput may result in
less time.
This phenomenon becomes even more challenging when the $\Delta_i$'s are unknown and when there are
multiple arms.

While the gaps and the ordering are also unknown in sequential and batch parallel BAI formulations,
the above considerations are a direct consequence of nonlinear scaling when allocating
resources, which is the focus of this work.

The second challenge motivates defining the following dynamic program
for $n$ arms which takes inputs $\{z_i\}_{i=2}^n\in\RR^{n-1}$ with
$z_2 \geq z_3\geq \dots \geq z_n$.
First define  $z_{n+1}=0$ and $\Tcal_{n+1}(\emptyset)=0$.
Then, recursively define $\Tcal_j:\RR^{n-j+1}\rightarrow\RR_+$ for $j=n,n-1,\dots,2$ as follows:
\begingroup
\allowdisplaybreaks
\begin{align*}
&\Tcal_j\left(\{z_i\}_{i=j}^{n}\right) =
\label{eqn:dp}\numberthis
\\
&\hspace{0.2in}\min_{k\in \{j,\ldots,n\}}
\Big( \lambda(k(z_j - z_{k+1})) + \Tcal_{k+1}\left(\{z_i\}_{i=k+1}^{n} \right) \Big).
\end{align*}
\endgroup
This program is best understood from the point of view of a planner who knows the
number of pulls necessary to eliminate arms $2,\dots,n$, i.e., $\{z_i\}_{i=2}^n$,
but is unaware of the ordering.
The planner's goal is to optimally schedule the pulls so as to minimize the total time.
Then, $\Tcal_j$ is the minimum time taken to eliminate the worst $n - j + 1$ arms.
For example, $\Tcal_n =
\lambda(nz_n)$, since the worst arm can be eliminated by pulling all $n$ arms $z_n$
times.
Similarly, for $\Tcal_j$, to eliminate the $n-j+1$ worst arms, the planner may first eliminate 
the first $n-j$ arms in time $\Tcal_{j+1}$,
and then pull each of the remaining arms $(z_j-z_{j+1})$ times;
alternatively, she may 
eliminate the first $n-j-1$ arms in time $\Tcal_{j+2}$,
and then pull each of the remaining arms $(z_j-z_{j+2})$ times etc.
There are $n-j+1$ such options, and the planner will choose the one that takes the
least time, as indicated in~\eqref{eqn:dp}.

Given~\eqref{eqn:dp},
we define $\Tstar$ as follows.
Recalling the definitions of the gaps $\Deltai$ from~\eqref{eqn:Deltai}, we have:
\begin{align*}
\Tstar = \Tcal_2\left(\{\Delta_i^{-2}\}_{i=2}^n  \right)
\label{eqn:Tstar}\numberthis
\end{align*}
Shortly, we will prove upper and lower bounds which depend on the gaps
via $\Tstar$ and differ only in factors that are doubly logarithmic in the gaps
and sub-polynomial in $n$.

\subsection{An Algorithm for the Fixed Confidence Setting}

Algorithm~\ref{alg:fixed_confidence}, called Adaptive Parallel Racing (\timebr), maintains
confidence intervals for the mean values,
and adaptively tunes the amount of parallelism by 
starting with a few pulls and then increasing the
level of parallelism during execution.
It operates over a sequence of rounds, indexed $r$, with the first round being allocated
$t_1 = \lambda(n)$ time and round $r$ being allocated $\beta^{r-1}t_1$ time.
Here, $\beta$ is an input parameter; the algorithm and the analysis work for any constant
value of $\beta>1$.
When $\beta=2$, this is akin to the doubling trick seen frequently in bandit settings,
except here we apply it in time-scale, instead of the number of pulls.

The algorithm maintains a subset $\Sr\subseteq[n]$ of surviving arms at round $r$.
At the beginning of each round, it pulls each arm in $\Sr$ a total of $q_r$ times,
such that $|\Sr|q_r$ is equal to the maximum number of arm pulls that can be executed in
$\beta^{r-1}t_1$ time.
At the end of each round, it constructs
confidence intervals $\{(L_i, U_i)\}_{i\in[n]}$ for the mean values $\{\mui\}_{i\in[n]}$  as
we will describe in~\eqref{eqn:confintervals}.
In lines~8 and~9,
it updates the set of surviving arms by eliminating those arms
whose upper confidence bounds are less than the highest lower confidence bound.
This strategy implies that
as the algorithm progresses, it spends more time per batch of pulls  by favoring throughput over information
accumulation.


\insertAlgoTimeBatchRacing

Our confidence intervals are based on the law of the iterated logarithm
and were first proposed by~\citet{jamieson2014lil}.
To describe them,
let $\Nir$ denote the number of times we have pulled arm $i$ in the first 
$r$ rounds, and $\muhatiNir$ denote the empirical mean of the samples collected.
If we pulled each arm $q_r$ times in round $r$,
we have:
\begingroup
\allowdisplaybreaks
\begin{align*}
\Nir = \sum_{s=1}^r q_s \indfone\{i\in S_s\},
\hspace{0.1in}
\muhatiNir  = \frac{1}{\Nir}\sum_{s=1}^r \sum_{j=1}^{q_s} X_{i,s,j}.
\end{align*}
\endgroup
Here, $\{X_{i,r,j}\}_{j=1}^{q_r}$ are the samples collected from arm $i$ in round $r$.
Next, let $D(N, \delta)=\sqrt{4\log(\log_2(2N)/\delta)/N}$
represent the uncertainty in using
 $\widehat{\mu}_{i,n}$ as an estimate for $\mui$.
Then, we can compute a confidence interval $(L_i(r, \delta), U_i(r, \delta))$ 
for arm $i$ as follows:
\begingroup
\allowdisplaybreaks
\begin{align*}
    L_i(r, \delta) &= \muhatiNir - D(N_i(r), \sqrt{\delta/(6n)})\\
    U_i(r, \delta) &= \muhatiNir + D(N_i(r), \sqrt{\delta/(6n)}).
\numberthis
\label{eqn:confintervals}
\end{align*}
\endgroup

\subsection{Upper Bound}

We now state our main result for the proposed algorithm.
Theorem~\ref{thm:fixed_confidence_runtime} shows that
Algorithm~\ref{alg:fixed_confidence} finds the best arm with probability greater than $1-\delta$
and bounds its execution time.


\insertprethmspacing
\begin{theorem}
\label{thm:fixed_confidence_runtime}
Assume $\lambda$ satisfies the assumptions in Section~\ref{sec:environment}.
Let $\beta \in (1,n]$.
Let $\omega = \sqrt{\delta/(6n)}$ and define $\bar{N}_i :=
1 + \left\lfloor 64\Delta_i^{-2}\log((2/\omega)\log_2(192\Delta_i^{-2}/\omega))\right\rfloor$
$\forall i \in [n]$.
Let $\Tcal_2$ be as defined in~\eqref{eqn:dp}.
With probability at least $1 - \delta$,
Algorithm~\ref{alg:fixed_confidence} outputs the best arm and the total execution time of the
algorithm $T$ satisfies:
\begingroup
\allowdisplaybreaks
\begin{align*}
    T &\;\leq\; 4\frac{\beta^{3 + 4\sqrt{\log_\beta(n)}}}{\beta - 1}
        \Tcal_2\left( \{\bar{N}_i\}_{i=2}^n\right) \\
    &\leq 512\frac{\beta^{3+4\sqrt{\log_\beta(n)}}}{\beta - 1}
    \left(\log\left(\frac{2}{\omega}\log_2\left(\frac{192}{\Delta^2_2\omega}\right)\right)\lor \frac{1}{64}\right)T^\star.
\end{align*}
\endgroup
\end{theorem}
While the first bound is tight,
the second bound  shows that that the runtime is bounded by essentially
$\Tstar\log(1/\delta)$; note $\log(1/\omega)\asymp\log(1/\delta)$.
All other terms are small:
$256 \beta^3/(\beta-1)$ is a constant as $\beta$ is a constant,
the additional dependence on the gaps $\{\Delta_i\}_i$ and $\delta$ is doubly logarithmic,
and $\beta^{4\sqrt{\log_\beta(n)}}$ is sub-polynomial, seen via the
following simple calculation.
Let $\alpha>0$. Then, 
\begin{align*}
\lim_{n\rightarrow\infty} \frac{\beta^{4\sqrt{\log_\beta(n)}}}{n^\alpha} = 
\beta^{\left(\scriptsize {\displaystyle\lim_{n\rightarrow\infty}} 4\sqrt{\log_\beta(n)} - \alpha
\log_\beta(n)\right)} =  0.
\end{align*}


\insertprethmspacing
\begin{example}
\label{exa:fixed_confidence}
Let us compare the above result with an algorithm which
operates sequentially taking $\scalefunc(1)$ time for each pull.
The  sequential algorithm of~\cite{jamieson2014lil}
terminates in time at most
$\lambda(1)(\bar{N}_2 + \sum_{i=2}^n \bar{N}_i) = \lambda(1)\sum_{i=2}^n
i(\bar{N}_i-\bar{N}_{i+1})$ w.p. at least $1 - \delta$.
On the other hand, Algorithm~\ref{alg:fixed_confidence} terminates in at most $C(n)\Tcal\left(
\{\bar{N}_i\}_{i=2}^n\right) \leq C(n)\sum_{i=2}^n \lambda(i(\bar{N}_i-\bar{N}_{i+1}))$,
where $C(n)\in \littleO({\rm poly}(n))$.
The second (looser) bound is obtained by considering one of the cases over which the minimum
is taken in the DP.
Additionally, by nonlinear scaling, we have
$\lambda(n)\leq n\lambda(1)$ (Lemma~\ref{lemma:scaling_factor});
once again, when the scaling is poor
this inequality is loose.
Even with these two loose inequalities, we have a runtime bound of
$C(n)\lambda(1)\sum_{i=2}^n i(\bar{N}_i-\bar{N}_{i+1})$ for
Algorithm~\ref{alg:fixed_confidence}, which is not 
worse than the sequential version up to lower order terms.

In general, however, $\scalefunc(1)$ may be very large in our problem set up.
We illustrate the advantages of  using an adaptive parallel strategy via an example, for which we 
consider a scaling function of the form
$\speedfunc(m) = m^{q}$, with $q\in[0, 1]$, 
which satisfies the assumptions in Section~\ref{sec:environment}.
When $q=1$, this corresponds to linear scaling, whereas when $q$ approaches zero,
the scaling becomes poor.
Moreover, let us assume that the $(\bar{N}_{i}  - \bar{N}_{i+1})$ is large for all $i$, so that
we have
$\Tcal_2\left( \{\bar{N}_i\}_{i=2}^n\right) = \sum_{i = 2}^{n} \left(i(\bar{N}_{i}  -
\bar{N}_{i+1})\right)^q$.
By  Theorem~\ref{thm:fixed_confidence_runtime}, ignoring constant and lower order factors,
we have that Algorithm~\ref{alg:fixed_confidence} terminates in
at most $\sum_{i = 2}^{n} \left(i(\bar{N}_{i}  - \bar{N}_{i+1})\right)^q$ time, while the sequential
algorithm terminates in
$\sum_{i=2}^n i(\bar{N}_i-\bar{N}_{i+1})$ time. As $q$ becomes smaller,
the difference between these bounds becomes larger.
Since the $\bar{N}_i - \bar{N}_{i+1}$ values are large, as per our example above,
this difference is quite pronounced.
\end{example}


\textbf{Proof Sketch:}
Designing algorithms in this setting is challenging, because adaptively finding the right level of
parallelism can be expensive. In particular, there are two modes of failure.
(i) Algorithms that take too long to ``ramp-up'', i.e. increase their parallelism, will spend
too much time with low throughput, which can slow progress if many arm pulls are necessary to
eliminate the next arms,
(ii) Alternatively, algorithms that have too much parallelism could potentially overshoot,
by over-pulling arms that could have been eliminated sooner.
Algorithm~\ref{alg:fixed_confidence} employs the doubling trick on \emph{time} to avoid
both of these situations:
 if all arms have not been eliminated, it multiplies the amount of time for the next round
by a factor $\beta$.
However,
while the doubling trick on the number of pulls admits a fairly straightforward analysis in most
bandit settings, the proof is significantly more challenging when it is used for a temporal criterion.
The key technical challenge is
to show that neither of the above two failure modes occurs frequently.
We first show that the ramp-up time 
Algorithm~\ref{alg:fixed_confidence} can be bounded within a constant factor of
$4\frac{\beta^3}{\beta - 1}$ of $T^\star$ in the runtime
upper bound. The trickier scenario is to show that overshooting is not significant, and a naive
analysis may result in a factor of $n$ being produced.
Our proof decomposes each stage, i.e., the rounds between arm eliminations,
into two
scenarios. In the first, the stage eliminates over a fraction $f$ of the arms, and in the second the
stage eliminates less than $f$. The first scenario can be expensive, as many arms that could have
been eliminated earlier were over-pulled.
However, we will carefully select $f$ to bound the number
of times this event can occur. The second scenario can happen many times, but we will show that each
occurrence does not add too much to the runtime, again by carefully selecting $f$.

\subsection{Lower Bound}
\label{sec:lower_bound_fc}

We conclude this section by demonstrating that the quantity $\Tstar$ in~(\ref{eqn:Tstar})
is fundamental
to this problem via a hardness result that matches the bound in
Theorem~\ref{thm:fixed_confidence_runtime}.
To state this theorem, let us define some quantities.
For a set of $n$ real-valued distributions $\theta = \{\thetabri\}_{i=1}^n$,
let $\muj$ denote the $j$\ssth mean of these distributions in descending order;
that is,
$\muone=\max_{i}\EE_{\thetabri}[X]\geq\mutwo\geq \dots \geq \muii{n-1}\geq
\mun=\min_{i}\EE_{\thetabri}[X]$.
Let $\Theta$ denote the following class of sub-Gaussian distributions with a well-defined best arm:
\begin{align*}
\Theta = \Big\{ \{\thetabri\}_{i=1}^n\,:\; &\muone > \mutwo \;\;\wedge\;
\\[-0.10in]
&\text{$\thetabrj$ is $1$-sub-Gaussian for all $j$}
\Big\}.
\end{align*}
Next, let $\Acaldelsf$
denote the class of algorithms which can identify the best arm with probability 
at least $1-\delta$ for a given scaling function $\scalefunc$ for \emph{any} set of distributions in
$\Theta$.
For an arbitrary algorithm $A\in\Acaldelsf$ executed on a problem $\theta\in\Theta$,
let $T(A,\theta)$ denote the time taken to stop.
The theorem below provides a nonasymptotic lower bound on the expectation of $T(A,\theta)$.

\insertprethmspacing
\begin{theorem}
Fix $\muone>\mutwo\geq\dots\geq \mun$.
Let $\Deltai$ be as defined in~\eqref{eqn:Deltai},
and $\Tstar$ be as defined in~\eqref{eqn:Tstar}.
Assume $\lambda$ satisfies the assumptions in Section~\ref{sec:environment}, and additionally 
for some $\alpha_1$,
$\alpha_1 m \leq \scalefunc(m)\;
\text{for all } m\geq 1.$
Then, there exists a set of distributions
 $\theta\in\Theta$ whose ordered mean values $\{\mui\}_i$
are such that for all $\delta\leq 0.15$,
\[
\inf_{A\in\Acaldelsf} \EE[T(A,\theta)] \geq 2 c_\scalefunc
\log\left(\frac{1}{2.4\delta} \right)\Tstar.
\]
Here, $c_\scalefunc$ is a constant that depends only on $\scalefunc$.
\label{thm:fixed_confidence_lowerbound}
\end{theorem}

The additional condition on $\scalefunc$ captures the practical notion that each arm pull requires a minimum amount of work $\alpha_1$ (fraction of resources $\times$ time) to execute: $\forall m \geq 0,\;\frac{1}{m}\lambda(m) \geq \alpha_1$.
The above theorem states that any algorithm which identifies the best arm with probability
at least $1-\delta$, has an expected runtime upper bounded by
$\bigOmega(\Tstar\log(1/\delta))$.
Modulo lower order terms, the RHS of the above lower bound matches the RHS in
the expression for the upper bound in Theorem~\ref{thm:fixed_confidence_runtime}.
This demonstrates that $\Tstar$ is a fundamental quantity in this setup.

It should be emphasized, however,
that the upper and lower bounds are not entirely comparable.
Theorem~\ref{thm:fixed_confidence_runtime} is a high-probability result,
guaranteeing that the best arm will be identified and the algorithm will terminate with probability at least $1-\delta$.
In contrast, Theorem~\ref{thm:fixed_confidence_lowerbound} lower bounds the \emph{expected} run time
of any algorithm that can identify the best arm with probability at least $1-\delta$.
This discrepancy is common in fixed confidence BAI settings,
with upper bounds tend to provide high-probability results for the number of arm pulls,
while lower bounds
are in expectation~\citep[e.g.,][]{jamieson2014lil,jun2016top,karnin2013almost,kaufmann2016complexity}.
Despite this discrepancy, no significant difference is yet to be observed, as is the case in
this work.
To our knowledge, only~\citet{kalyanakrishnan2012pac} upper bound the expected number of
pulls.

%
\textbf{Comparison to Prior Work:}
It is worth comparing the above results with prior work in the fixed confidence setting.
First, for the upper bound, our algorithm uses similar confidence intervals
to~\citet{jamieson2014lil} and~\citet{jun2016top} who study the sequential
and batch parallel settings respectively.
However, unlike~\citet{jun2016top}, in our setting, we also need to choose the amount of resources
to allocate for each pull.
This determines the amount of parallelism to handle the tradeoff between information
accumulation and throughput and depends on the scaling function $\scalefunc$.
More importantly, much of our analysis in Appendix~\ref{sec:app_fixedconf}
is invested in managing this tradeoff which is not encountered in their settings.
Similarly,
for the lower bound, while we rely on some hardness results from~\citet{kaufmann2016complexity},
their result only captures the sample complexity of the problem and does not account for how
resource allocation strategies may affect the time taken to collect those samples.
The novelty of our work, relative to the above works, is further highlighted by the fact
both the lower and upper bounds are given by a dynamic program which, as explained in the
beginning of this section, characterises the tradeoff beween throughput and information
accumulation.

\section{Fixed Deadline Setting}
\label{sec:fixedbudget}

While the primary focus of this paper is the fixed confidence setting, we also provide a simple
algorithm for the fixed deadline version of this problem.
Formally, we assume the same environment as described in
Section~\ref{sec:environment}, but now we have a time deadline $T$
and wish to maximize the probability of finding the best arm under this deadline.

Our algorithm builds on the
sequential-halving (SH) algorithm of~\citet{karnin2013almost}.
We begin with a brief review of SH in the sequential setting where
we are given a budget on the \emph{number of arm pulls}.
SH divides this budget into $\log_2(n)$ equal stages.
In the first stage, it pulls all
arms an equal number of times and eliminates the bottom half of the arms, i.e., those arms whose
empirical mean fall within in the bottom $n/2$ when ranked.
It continues in this fashion for each subsequent stage,  eliminating half of the surviving arms,
until there is one arm left at the end of $\log_2(n)$
stages.

Now consider
a naive extension of this algorithm that divides the \emph{time budget} 
$T$ into $\log_2(n)$
stages and pulls the surviving set of arms maximally before eliminating half of them---for
simplicity, we will refer to this time-scale version as SH from now on.
However, if the scaling
function is sublinear, then splitting the time budget into larger stages can better take advantage
of parallelism to execute more arm pulls and hence do better than an adaptive algorithm.
For example, assume there are $n=4$ arms, let $\scalefunc(m) = m^{\nicefrac{1}{4}}$,
and let the deadline be $T=4$. 
If the budget was split into $\log_2(4)=2$ stages, where we eliminate two arms in the first stage
and one in the second.
Then, in each stage, $\scalefuncinv(2) = 16$
arm pulls can be executed with four pulls per arm in the first stage and eight pulls per arm in the
second.
If instead, we executed all arms in  a single stage, then $\scalefuncinv(4) = 256$ arm pulls
can be executed, with $64$ pulls per arm.
In the first strategy, attempting to accumulate information 
reduces throughput significantly.
We should prefer the second option, as every arm is
pulled more than it would be under the original SH strategy.
In this section, we describe an algorithm
that uses $\scalefunc$ to balance between throughput and information
accumulation by allocating resources to promising arms while maintaining 
high throughput.

\insertAlgoStageCombSeqHalving
Our algorithm, outlined in Algorithm~\ref{alg:fixed_budget}, takes a hyperparameter $k$.
It splits $T$ into $\lceil \log_{2^k}(n) \rceil$ stages, and pulls surviving arms maximally in each
stage. Observe that using $k=1$ corresponds to running SH.
It keeps the $\left \lceil \frac{1}{2^k} \right \rceil$ fraction of arms with highest empirical
mean. Intuitively, increasing $k$ will allow more time per stage, so throughput can be increased in each stage. However, increasing the time per stage will increase the time to obtain results and reduce opportunities to reallocate resources to promising arms.
We propose using $k=\kstar$, obtained via the following optimization problem:
\begingroup
\allowdisplaybreaks
\begin{align*}
    \numberthis\label{eq:xk}
    \kstar 
    &= \arg\max_{k \in \{1,\ldots,\lceil \log_2(n) \rceil\}} x(k),\\
    \quad\, &\text{s.t.} \,\quad\textbf{}
    x(k) = \left \lfloor \frac{\lambda^{-1}\left(\frac{T}{ \lceil \log_{2^{k}}(n) \rceil}\right)}{2^{k \left\lceil\log_{2^k}(n)\right\rceil}(2^k - 1) } \right \rfloor
\end{align*}
\endgroup
$\kstar$ can be computed in practice
since all parameters are known; moreover, this can be done
inexpensively by evaluating $x(k)$ for all $\lceil \log_2(n) \rceil$ values for $k$. We show
later that $\kstar$ increases as the scaling deteriorates, which means we should
prioritize throughput over information accumulation.

\insertFigFixedConf

Ignoring rounding effects for simplicity, the numerator in~\eqref{eq:xk} indicates the maximum number of pulls that can be completed during the stage. In the
denominator, the first term is $n$, the number of arms.
If a certain $k'$ allows the algorithm to pull the arms at least $\sum_{i=0}^{k'-1} 2^i x(1)$
times in the first stage, then intuitively we should prefer $k=k'$ over $k=1$, as all arms are pulled more times in
the first stage of the former than any arm in the first $k'$ stages of the latter. Observe now that this occurs if $x(k') \geq x(1)$.
Extrapolating this argument, we have that we should choose the value $k$ which maximizes $x(k)$.
The following theorem bounds the probability of error for Algorithm~\ref{alg:fixed_budget}.

\insertprethmspacing
\begin{theorem}\label{thm:fixed_budget_error_prob}
Assume $\scalefunc$ satisfies the assumptions in Section~\ref{sec:environment}.
Algorithm~\ref{alg:fixed_budget} run with some
$k \in \NN_+$ identifies the best arm in time at most
$T$ with probability at least:
\begin{align}
    1 -  3\lceil\log_2(n)\rceil \exp \left(-\frac{nx(k)}{8H_2}\right),
    \label{eq:fb_success_prob}
\end{align}
where $H_2 = \max_{i\neq 1} i\Delta^{-2}_i$ and $\Deltai$ is as defined in~\eqref{eqn:Deltai}.
\end{theorem}
In Appendix~\ref{app:fixed_budget}, we show that the success probability of SH is obtained by
setting $k=1$ in the expression for $x(k)$, and,
by choosing $k=\kstar$, this probability is always better than SH.
Unfortunately, is is not possible to simplify this further without additional assumptions on
$\speedfunc$.
Therefore, 
in order to illustrate the gains in using Algorithm~\ref{alg:fixed_budget} over SH,
we consider a specific example.

\insertprethmspacing
\begin{example}
Let us consider $\scalefunc$ of the form $\speedfunc(m) = m^{q}$ for $q\in(0,1]$.
The scaling becomes poor as $q$ approaches zero (recall footnote~\ref{ftn:mqexample}).
In Appendix~\ref{app:fbexample}, we show that \emph{\scsh} is quantitatively better than SH;
moreover, this difference is magnified as
the scaling becomes poor, i.e., $q$ approaches zero.
We sketch the argument here.
First we show, 
\begin{align*}
    x(k) &\geq \frac{1}{4^{k}n}\left(\frac{Tk}{ 2\log_{2}(n)}\right)^{1/q}.
\end{align*}
Therefore, the number of arm pulls per round increases as $q$ decreases and as $k$ increases (after multiplying $x(k)$ by $2^k - 1$ to de-normalize). 
Let $p_{\rm ssh, k}$ be the error probability of combining $k$ stages in
Algorithm~\ref{alg:fixed_budget}. Applying Theorem~\ref{thm:fixed_budget_error_prob},
\begingroup
\allowdisplaybreaks
\begin{align*}
    p_{\rm ssh, k}
    &= C\lceil\log_2(n)\rceil\cdot\exp \left(-D \cdot \left(\frac{k^{1/q}}{4^k}\right)\right),
\end{align*}
\endgroup
where $C$ and $D$ are constants that do not depend on $k$. When the scaling is sufficiently poor,
i.e., when $1/q$ is large, we have
$\left(\frac{k^{1/q}}{4^k}\right) > \left(\frac{k'^{1/q}}{4^{k'}}\right)$
for any $k > k'$. So, the exponentially decaying term will favor larger $k$ as the scaling becomes
more poor. So, as scaling deteriorates, the error probability is lower for larger $k$ values, which
prioritizes throughput over information accumulation and resource reallocation. However, if $k$ is
too large for a given $q$, then the $4^k$ term will dominate, which occurs when the algorithm isn't
reallocating enough resources to promising arms.
\end{example}
\vspace{-0.1in}


\section{Simulations}

We present an experimental evaluation in the fixed confidence setting.
We present experiments in the fixed deadline setting and some
experimental details in Appendix~\ref{app:experiments}.

\textbf{Baselines:} We compare \timebrs to Batch Racing~\citep{jun2016top} for batch
parallel BAI.
We apply Batch-Racing($m$) with different batch sizes $m$, but with a fixed amount used throughout
one execution of the algorithm.
\vspace{-0.1in}

\textbf{Results:}
We evaluate Algorithm~\ref{alg:fixed_confidence} with $\beta=2$ on a synthetic domain consisting of
$n=16$
Bernoulli arms, $\delta=0.1$, $\lambda(m) = m^{q}$, and with different values of
$\Deltatwo=\muone-\mutwo$. For a given $\Delta_2$, the
arm means have linearly interpolated values, $\mu_i = 0.9 - \Delta_2 - \frac{(0.9 - 0.1)(i - 2)}{15}$ for $i \geq 2$ and
$\mu_1 = 0.9$.
This sets the best arm to $0.9$ and sets the remaining means by linearly
 varying $\Delta_2$ from $0.01$ to $0.5$. We evaluate
all algorithms ten times on each setting.
All algorithms almost always identify
the best arm in the experiments.
\timebr{} consistently does better than Batch Racing with a fixed amount of parallelism,
and does as well as the best task-tuned batch size on the problem.
See Figure~\ref{fig:fixed_conf_experiments}.
\vspace{-0.2in}


\section{Summary}
We consider a novel setting for BAI, where arm
pulls can be parallelized by dividing a fixed set of resources across them.
While allocating more resources to a pull produces results
sooner, this may result in lower throughput overall.
So, algorithms must trade off between \emph{information accumulation},
which allows us to invest resources in more promising candidates in future iterations,
and \emph{throughput}, which increases
the overall number of samples.
One avenue for future work is to study lower bounds for the fixed deadline setting.


\bibliography{refs,kky}

\bibliographystyle{icml2021}

\newpage
\onecolumn
\appendix

\section{Properties of the Scaling Function $\lambda$}

We will prove and discuss some basic properties of $\lambda$ that will provide intuition for the setting and also be used in the later proofs.

\insertprethmspacing
\begin{lemma}\label{lemma:scaling_factor}
Let $\alpha\in\RR_+$.
If $\alpha \leq 1$, then $\lambda(\alpha m) \geq \alpha \lambda(m)$. 
If $\alpha \geq 1$, then $\lambda(\alpha m) \leq \alpha \lambda(m)$.
\end{lemma}

\begin{proof}
By concavity,
\begin{align*}
    \lambda(\alpha m) &= \lambda(\alpha m + (1 - \alpha) 0)
    \geq \alpha \lambda(m)
\end{align*}
To prove the second part, observe that:
\begin{align*}
    \lambda(m) &= \lambda\left(\frac{1}{\alpha}\alpha m\right)
    \geq \frac{1}{\alpha} \lambda (\alpha m) \implies \lambda(\alpha m) \leq \alpha \lambda(m)
\end{align*}
The first part is applied above since $\alpha \geq 1$.
\end{proof}

If we divide both sides of each inequality by $\alpha m$, then each side represents the average time per arm pull when the system executes $\alpha m$ and $m$ arm pulls, respectively. Intuitively, the average time per arm pull should decrease as the number of parallel arm pulls increases.

\insertprethmspacing
\begin{lemma}\label{lemma:splitting}
For all $m_1, m_2 > 0$,
\begin{align*}
\lambda\left(m_1+m_2\right) &\leq \lambda\left(m_1\right) +
\lambda\left(m_2\right),
\end{align*}
\end{lemma}
\begin{proof}
We can apply Inequality~\ref{eqn:speedfuncassumption}:
\begin{align*}
    \frac{\lambda(m_1 + m_2) - \lambda(m_1)}{m_2} &\leq \frac{\lambda(m_2) - \lambda(0)}{m_2} \implies \lambda\left(m_1+m_2\right) \leq \lambda\left(m_1\right) + \lambda\left(m_2\right)
\end{align*}
\end{proof}

This property implies that the fastest way to schedule $m$ arm pulls is to run them all in parallel instead of executing $m_1$ and then $m - m_1$.

\section{Proof of the Upper Bound in the Fixed Confidence Setting}
\label{sec:app_fixedconf}


We will upper bound the time complexity of Algorithm~\ref{alg:fixed_confidence} and show that it
outputs the top arm with probability at least $1 - \delta$. Bounding the time complexity is
challenging, as we must show that the algorithm does not waste too much time as it ramps up its
parallelism during execution. If it takes too long to find an appropriate level of parallelism, then
it may be significantly longer than $T^\star$. However, if it ramps up too fast, it may
\emph{overshoot} by unnecessarily pulling arms that should have been eliminated much sooner. For
example, if arm $i$ could safely be eliminated after $10$ more pulls, but the algorithm pulls each
surviving arm $1000$ times in parallel in a round, then the algorithm has wasted time pulling arm $i$ more
than necessary.
As the pulls are executed in parallel, any information is only acquired
after they all finish (at the same time). This behavior can be costly if a particular round of Algorithm~\ref{alg:fixed_confidence} pulls many arms that could be eliminated soon excessively.
We will use an analysis that carefully defines these overshooting events and
shows that the extremely costly events do not occur often, and the less costly ones cannot
significantly impact the overall runtime.

\subsection{Confidence Intervals}

In this section,
we will state some results from~\citet{jun2016top} to bound the error probability.
We will first define some notation.
Let $\mathcal{E}_i(\delta) = \{\forall r\geq 1, L_i(r,\delta) \leq \mu_i \leq U_i(r,\delta)\}$, which is the event that the confidence bounds capture the true mean of arm $i$ at all rounds of the algorithm. Let $\omega = \sqrt{\delta/(6n)}$. Define $\bar{N}_i := 1 + \left\lfloor 64\Delta_i^{-2}\log((2/\omega)\log_2(192\Delta_i^{-2}/\omega))\right\rfloor$. 
We will use the following results from~\citet{jun2016top} to show that the confidence intervals
in~\eqref{eqn:confintervals} trap the true means.

\insertprethmspacing
\begin{lemma}\label{lemma:confidence}
[\citet{jun2016top}, Lemma 1] (non-asymptotic law of the iterated logarithm). Let $X_1,X_2,\ldots$ be i.i.d. zero-mean sub-Gaussian random variables with scale $\sigma > 0$; i.e. $\mathbb{E} \exp{\left(\lambda X_i\right)} \leq \exp{\left(\frac{\lambda^2\sigma^2}{2}\right)}$. Let $\omega \in (0, \sqrt{1/6})$. Then,
\begin{align*}
\mathrm{Pr}\left[\forall \tau \geq 1, \left|\sum_{s=1}^\tau X_s\right| \leq 4\sigma \sqrt{\tau \log{\log_2{2\tau}}/\omega}\right] \geq 1 - 6\omega^2
\end{align*}
\end{lemma}

This lemma is a simplification of Lemma 1 from~\citet{jamieson2014lil}. Observe that a random variable bounded $a.s.$ in $[0, 1]$ is sub-Gaussian with scale $1/2$. By Lemma~\ref{lemma:confidence},  $\mathrm{Pr}\left[\cap_{i=1}^n \mathcal{E}_i(\delta)\right] \geq 1 - \delta$. Hereafter, we will assume $\cap_{i=1}^n \mathcal{E}_i(\delta)$ and show that the algorithm always outputs the top arm in this event and bound its runtime. Furthermore, we introduce the following result from~\citet{jun2016top}.

\insertprethmspacing
\begin{lemma}\label{lemma:runtime}
[\citet{jun2016top}, Lemma 2] Assume $\cap_{i=1}^n \mathcal{E}_i(\delta)$. In Algorithm~\ref{alg:fixed_budget}, let $N'(r) = \min_{i\in S_r} N_i(r)$. Then,
\begin{align}
    &\forall r, \forall i > 1, \left(N'(r)\geq\bar{N}_i \implies U_i(r,\delta) < \max_{j\in S_r} L_j(r,\delta)\right)\\
    &\forall r, i = 1, \left(N'(r)\geq\bar{N}_i \implies L_i(r,\delta) > \max_{j\in S_r}^{(2)} U_j(r,\delta)\right)
\end{align}
\end{lemma}
So, as long as $\mathbb{P}(\cap_{i=1}^n \mathcal{E}_i) \geq 1 - \delta$, the algorithm will output the correct set of arms after each surviving arm has been pulled $\bar{N}_i$ times with probability at least $1 - \delta$.

\subsection{Proof of Theorem~\ref{thm:fixed_confidence_runtime}}
In this proof, we will assume $\cap_{i=1}^n \mathcal{E}_i(\delta)$, which we know by Lemma~\ref{lemma:confidence} to have probability at least $1 - \delta$.

\textbf{Correctness:}
If an arm $i$ is returned by the algorithm as the best arm, then the other surviving arms have true means less than $\mu_i$, because $\cap_{i=1}^n \mathcal{E}_i(\delta)$ implies that
the confidence bounds capture the true means, and the algorithm will only accept arm $i$ if its
lower confidence bound (LCB) 
is greater than the other arms' upper confidence bounds (UCB).
The same argument follows for arms that are rejected. So, when
$\cap_{i=1}^n \mathcal{E}_i(\delta)$ occurs, arms cannot be erroneously returned or rejected. So, the algorithm outputs the correct arm with probability at least $1 - \delta$.

\textbf{Time Complexity:}
Recall that the dynamic program $\Tcal_2\left(\{\bar{N}_i\}_{i=2}^n \}  \right)$ is the runtime of an algorithm that operates in stages to eliminate arms. This algorithm knows the arm gaps, but not the arm orderings. We will refer to this algorithm hereafter as the \textit{oracle algorithm}. The oracle algorithm consists of $\widetilde{r}\leq n - 1$ stages, and eliminates $\widetilde{n}_i$ arms in stage $i$, which is possible with probability at least $1 - \delta$. Define $\bar{N}_{n+1} = 0$.
Also define $\widetilde{n}_{i, \rm tot} = \sum_{j=1}^i \widetilde{n}_{j}$ (number of arms eliminated by round $i$, with $\widetilde{n}_{0, \rm tot}=0$),  $\Delta\widetilde{N}_{i} = \bar{N}_{n - \widetilde{n}_{i, \rm tot} + 1} - \bar{N}_{n - \widetilde{n}_{i-1, \rm tot} + 1}$ (pulls per surviving arm in oracle stage $i$), and $\widetilde{T}_i^\star = \lambda\left((n - \widetilde{n}_{i-1, \rm tot})\Delta\widetilde{N}_{i}\right)$ (time for oracle stage $i$).
Then,
\begingroup
\allowdisplaybreaks
\begin{align*}
    \widetilde{T}^\star &= \Tcal_2\left(\{\bar{N}_i\}_{i=2}^n \}  \right)\\
    &= \sum_{i=1}^{\widetilde{r}} \lambda\left((n - \widetilde{n}_{i-1, \rm tot})(\bar{N}_{n - \widetilde{n}_{i, \rm tot} + 1} - \bar{N}_{n - \widetilde{n}_{i-1, \rm tot} + 1})\right)\\
    &= \sum_{i=1}^{\widetilde{r}} \lambda\left((n - \widetilde{n}_{i-1, \rm tot})\Delta\widetilde{N}_{i}\right)\\
    &= \sum_{i=1}^{\widetilde{r}} \widetilde{T}_i^\star
\end{align*}
\endgroup
Now,
assign each pull in oracle stage $1$ a unique index from $\{1,\ldots,\Delta\widetilde{N}_1\}$. For
each pull in oracle stage $2$, assign a unique index from
$\{\Delta\widetilde{N}_1+1,\ldots\Delta,\widetilde{N}_2\}$. Assign indices for the remaining pulls
accordingly. Now, define the function $\widetilde{R}$ that outputs which oracle stage a pull index corresponds
to. That is,
\begin{align*}
    \widetilde{R}(q) &=
    \begin{cases}
        \max_{i \in [\widetilde{r}]} i \text{ s.t. } q \leq \bar{N}_{n - \widetilde{n}_{i, \rm tot} + 1} &\mbox{if } q \leq \bar{N}_2 \\
        \widetilde{r} + 1 & \mbox{if } q > \bar{N}_2
\end{cases}
\end{align*}
In general, symbols with a tilde are associated with the oracle algorithm in this proof.

Define stage $1$ of Algorithm~\ref{alg:fixed_confidence} (and not the oracle algorithm) to be the rounds of the algorithm up to and including the round that results in the first arm elimination. Define stage $2$ to be from the next round until the next eliminations unless multiple arms were eliminated at the same time as the first elimination. If the second arm was eliminated at the same time as the first arm, define a dummy stage $2$ that takes $0$ rounds. Define additional dummy stages for each additional arm that was eliminated in the first elimination round. Define stage $i$ similarly. Observe that there are always $n-1$ stages now, so during stage $i$, $n - i + 1$ arms remain.

Let $\ell_i$ be the number of rounds in stage $i$ and let $\bar{\ell}_i = \sum_{j=1}^i \ell_j$ (rounds in first $j$ stages). Similarly, let $\bar{q}_j = \sum_{i=1}^j q_i$ (pulls per surviving arm after round $j$) and $\bar{t}_j = \sum_{i=1}^j t_i$ (sum of time allocations in first $j$ rounds) where $t_i = \beta^{i-1}t_1$. Define $T_i = \sum_{j=\bar{\ell}_{i-1} + 1}^{\bar{\ell}_i} \lambda(|S_j|q_j) \leq \sum_{j=\bar{\ell}_{i-1} + 1}^{\bar{\ell}_i} t_j$ (time for stage $i$).

Let us discuss a few easily-verified facts about the algorithm before proceeding.

\insertprethmspacing
\begin{fact}\label{lemma:time_geom}
Suppose Algorithm~\ref{alg:fixed_confidence} takes $l$ rounds. Then its runtime is at most $\frac{\beta^{l}}{\beta - 1}t_1$.
\end{fact}
The above property suggests that adding an additional round to the algorithm will increase its runtime by a factor of $\beta$.

\insertprethmspacing
\begin{fact}\label{lemma:time_facts}
$\lambda(q_j|S_j|) = \lambda\left(\left\lfloor\frac{\lambda^{-1}(t_j)}{|S_j|}\right\rfloor|S_j|\right) \leq t_j = \beta^{j-1}t_1 \leq \lambda(2q_{j}|S_j|) \leq 2\lambda(q_{j}|S_j|)$
\end{fact}
The second to last inequality follows because $q_j \geq 1$ for all $j \geq 1$.

Stage $i$ could fall into one of $5$ cases. We will first handle $3$ cases which are easier to
analyze, and then turn to two more challenging cases.

\textbf{Case 1:}
Suppose $\ell_i \geq 2$. Define $\widetilde{s}_0 = \widetilde{R}(\bar{q}_{\bar{\ell}_{i} - 2} + 1)$ and $\widetilde{s}_f = \widetilde{R}(\bar{q}_{\bar{\ell}_{i}-1})$.
This means that stage $i$ took at least $2$ rounds and the pull indices in its second to last round of stage $i$ fall between the oracle algorithm's stage $\widetilde{s}_0$ and $\widetilde{s}_f$, respectively.
Because the pull indices of round $\bar{\ell}_{i} - 1$ fall within the oracle algorithm's stages $\widetilde{s}_0,\ldots,\widetilde{s}_f$, the oracle algorithm must have more pulls per arm across these stages than round $\bar{\ell}_{i} - 1$ of Algorithm~\ref{alg:fixed_confidence} does. That is,
\begin{align*}
    q_{\bar{\ell}_{i} - 1}&\leq  \sum^{\widetilde{s}_f}_{k=\widetilde{s}_0} \Delta\widetilde{N}_{k}
\end{align*}
This implies that:
\begin{align*}
    (n - i + 1)q_{\bar{\ell}_{i}-1} &\leq  \sum^{\widetilde{s}_f}_{k=\widetilde{s}_0} (n - i + 1) \Delta\widetilde{N}_{k}\\
    &\leq  \sum^{\widetilde{s}_f}_{k=\widetilde{s}_0} (n - \widetilde{n}_{k-1,\rm tot}) \Delta\widetilde{N}_{k}\label{eq:case1pullbound}\numberthis
\end{align*}
The second inequality uses the fact that $\widetilde{n}_{\widetilde{s}_0 -1 , \rm tot} \leq \widetilde{n}_{\widetilde{s}_f -1 , \rm tot}\leq i - 1$. Each surviving arm after round $\bar{q}_{\bar{\ell}_i - 1}$ of Algorithm~\ref{alg:fixed_confidence} has been pulled at least $\bar{N}_{n - \widetilde{n}_{\widetilde{s}_f - 1} + 1}$ times (because $\widetilde{s}_f = \widetilde{R}(\bar{q}_{\bar{\ell}_{i}-1})$), guaranteeing that at least $\widetilde{n}_{\widetilde{s}_f-1}$ arms were eliminated after this round (Lemma~\ref{lemma:runtime}).

Let us now use these facts to bound the runtime of the first $i$ stages of Algorithm~\ref{alg:fixed_confidence}.
\begingroup
\allowdisplaybreaks
\begin{align*}
    \sum_{k=1}^i T_k &\leq \sum_{r=1}^{\bar{\ell}_i} t_r\\
    &= \sum_{r=1}^{\bar{\ell}_i} \beta^{r - 1}t_{1}\\
     &\leq \frac{\beta^2}{\beta - 1}\beta^{\bar{\ell}_i - 2}t_{\bar{\ell}_{1}}\\
     &\leq 2\frac{\beta^2}{\beta - 1}\lambda((n - i + 1)q_{\bar{\ell}_{i}-1})\\
    &\leq 2\frac{\beta^2}{\beta - 1}\lambda \left(\sum^{\widetilde{s}_f}_{k=\widetilde{s}_0} (n - \widetilde{n}_{k-1,\rm tot}) \Delta\widetilde{N}_{k}\right)\\
    &\leq 2\frac{\beta^2}{\beta - 1}\sum^{\widetilde{s}_f}_{k=\widetilde{s}_0} \lambda((n - \widetilde{n}_{k-1,\rm tot}) \Delta\widetilde{N}_{k})\\
    &= 2\frac{\beta^2}{\beta - 1} \sum^{\widetilde{s}_f}_{k=\widetilde{s}_0} \widetilde{T}_k^\star
\end{align*}
\endgroup
The first inequality uses Fact~\ref{lemma:time_facts}. The second inequality uses Fact~\ref{lemma:time_geom}. The third inequality again uses Fact~\ref{lemma:time_facts}. The fourth inequality uses Inequality~\ref{eq:case1pullbound}, and the fifth inequality applies Lemma~\ref{lemma:splitting}. The last equality uses the definition of $\widetilde{T}^\star_k$.

\textbf{Case 2:}
Suppose $\ell_i = 1$ and suppose that $\widetilde{R}(\bar{q}_{\bar{\ell}_{i-1}} + 1) = \widetilde{R}(\bar{q}_{\bar{\ell}_{i}}) = \widetilde{s}$. That is, stage $i$ takes a single round and has pull indices that are completely contained in oracle stage $k$. By almost the same argument as above:

\begin{align*}
    q_{\bar{\ell}_{i}}&\leq  \Delta\widetilde{N}_{\widetilde{s}}
\end{align*}
As before, this implies
\begin{align*}
    (n - i + 1)q_{\bar{\ell}_{i}} &\leq  (n - i + 1) \Delta\widetilde{N}_{\widetilde{s}}\\
    &\leq  (n - \widetilde{n}_{\widetilde{s}-1,\rm tot}) \Delta\widetilde{N}_{\widetilde{s}}\label{eq:case2pullbound}\numberthis
\end{align*}
The second inequality uses the fact that $\widetilde{n}_{\widetilde{s}-1, \rm tot}\leq i - 1$. Each surviving arm after round $\bar{q}_{\bar{\ell}_1 - 1}$ of Algorithm~\ref{alg:fixed_confidence} has been pulled at least $\bar{N}_{n - \widetilde{n}_{\widetilde{s} - 1} + 1}$ times (because $\widetilde{s} = \widetilde{R}(\bar{q}_{\bar{\ell}_{i}-1} + 1)$), guaranteeing that at least $\widetilde{n}_{\widetilde{s}-1}$ arms were eliminated after this round (Lemma~\ref{lemma:runtime}). Note that this implies that overshooting did not occur. Again, as before

\begin{align*}
    \sum_{k=1}^i T_k &\leq \sum_{r=1}^{\bar{\ell}_i} t_r\\
    &= \sum_{r=1}^{\bar{\ell}_i} \beta^{r - 1}t_{1}\\
     &\leq \frac{\beta}{\beta - 1}\beta^{\bar{\ell}_i - 1}t_{\bar{\ell}_{1}}\\
     &\leq 2\frac{\beta}{\beta - 1}\lambda((n - i + 1)q_{\bar{\ell}_{i}})\\
    &\leq 2\frac{\beta}{\beta - 1}\lambda \left((n - \widetilde{n}_{k-1,\rm tot}) \Delta\widetilde{N}_{\widetilde{s}}\right)\\
    &= 2\frac{\beta}{\beta - 1} \widetilde{T}_{\widetilde{s}}^\star
\end{align*}

This result is missing a factor of $\beta$ compared to Case 1, since in the prior case we used the second to last round in stage $i$ to achieve the bound on the first $i$ stages, whereas here we can use the last (and only) round.

\textbf{Case 3:}
Suppose $\ell_i = 0$. This means that the $i$-th elimination occurred in the same round as the $i-1$-th elimination. The time taken for these stages is $0$.


\textbf{Discussion of Cases 1-3:}
Let stage $i_f$ be the last stage of Algorithm~\ref{alg:fixed_confidence} that falls into either cases 1 or 2 above. Then, stages $1,\ldots,i_f$ take at most $2\frac{\beta^2}{\beta - 1}\widetilde{T}^\star$ time. The remaining stages all consist of a single round, and they may potentially pull arms excessively due to overshooting (excessively pulling arms that would have been eliminated after a few pulls in the round). We must now argue that overshooting does not result in a significant increase in runtime. Observe that naively bounding the up to $n-1$ remaining stages by a factor of $\beta^{n-1}$ results in a very expensive runtime. However, we will use a more aggressive analysis by defining events where excessive overshooting occurs, and bounding the number of these events.

We will discuss $2$ cases for the remaining stages now. If some stage $i \in [i_f]$ falls in cases 4 and 5, we already have a runtime bound for the first $i_f$ cases and therefore do not need to bound the runtime of these stages.


Let $f = \beta^{-\sqrt{\log_\beta(n)}}$. Ignore the very last stage for now, which can only add a factor of $\beta$ to the total runtime. Each of the remaining stages fall into one of two cases:

\textbf{Case 4:}
Define $\widetilde{s}_0 = \widetilde{R}(\bar{q}_{\bar{\ell}_{i - 1}} + 1)$ and $\widetilde{s}_f = \widetilde{R}(\bar{q}_{\bar{\ell}_{i}})$. Suppose stage $i$ eliminates at least $1 - f$ fraction of the surviving arms, $\ell_i = 1$, and $\widetilde{s}_0 < \widetilde{s}_f$. That is, stage $i$ of Algorithm~\ref{alg:fixed_confidence} takes a single round, $S_{\bar{\ell}_{i} + 1} \leq f S_{\bar{\ell}_{i}}$, and has pull indices that span multiple oracle stages. This can occur at most $\left\lceil\sqrt{\log_\beta(n)}\right\rceil \leq 2\sqrt{\log_\beta(n)}$ times, because $\beta \leq n$. Each one of these occurrences adds a factor of $\beta$ to the runtime of the stages prior to it. The worst increase occurs if this case happens at the very last stages. So, even if this case occurred the maximum number of times after all other stages, this would add a factor of $\beta^{2\sqrt{\log_\beta(n)}}$ to the runtime by Fact~\ref{lemma:time_geom}.

\textbf{Case 5:}
Define $\widetilde{s}_0 = \widetilde{R}(\bar{q}_{\bar{\ell}_{i - 1}} + 1)$ and $\widetilde{s}_f = \widetilde{R}(\bar{q}_{\bar{\ell}_{i}})$. Suppose stage $i$ eliminates less than $1 - f$ fraction of the arms, $\ell_i = 1$ and $\widetilde{s}_0 < \widetilde{s}_f$. That is, stage $i$ of Algorithm~\ref{alg:fixed_confidence} takes a single round, $S_{\bar{\ell}_{i} + 1} > f S_{\bar{\ell}_{i}}$, and has pull indices that span multiple oracle stages (which could mean overshooting). This can occur up to $n - 1$ times, and at the end of stage $i$, at least $f(n - i + 1) \leq (n - \widetilde{n}_{y -1 , \rm tot}) \leq (n - \widetilde{n}_{x -1 , \rm tot}) $ arms remain (Lemma~\ref{lemma:runtime}). Therefore,

\begin{align*}
    (n - i + 1)q_{\bar{\ell}_{i}} &\leq  \sum^{y}_{k=x} (n - i + 1) \Delta\widetilde{N}_{k}\numberthis\label{eq:case4pullbound}
\end{align*}
We can use this fact as before to show the following:
\begin{align*}
    T_i &= \lambda((n - i + 1)q_{\bar{\ell}_i})\\
    &\leq \sum^{\widetilde{s}_f}_{k=\widetilde{s}_0} \lambda((n - i + 1) \Delta\widetilde{N}_{k})\\
    &\leq \sum^{\widetilde{s}_f}_{k=\widetilde{s}_0} \lambda\left(\frac{1}{f}(n - \widetilde{n}_{k-1,\rm tot}) \Delta\widetilde{N}_{k}\right)\\
    &\leq \frac{1}{f}\sum^{\widetilde{s}_f}_{k=\widetilde{s}_0} \lambda\left((n - \widetilde{n}_{k-1,\rm tot}) \Delta\widetilde{N}_{k}\right)\\
    &\leq \frac{1}{f}\sum^{\widetilde{s}_f}_{k=\widetilde{s}_0} \widetilde{T}_k^\star
\end{align*}
The first inequality follows by applying Inequality~\ref{eq:case4pullbound} followed by Lemma~\ref{lemma:splitting}. The third inequality applies Lemma~\ref{lemma:scaling_factor}, since $\beta \leq n \implies 1/f \geq 1$.


While this case can happen up to $n-1$ times, each stage in the oracle algorithm can only be covered by at most $2$ such stages from Algorithm~\ref{alg:fixed_confidence}, because each stage in this case has pull indices intersecting with at least $2$ oracle stages. So, even if all stages fell in this case, this would add $\frac{2}{f}\widetilde{T}^\star = 2\beta^{2\sqrt{\log_\beta(n)}}\widetilde{T}^\star$ to the runtime.


\textbf{Putting together the pieces:}
The runtime up to stage $i_f$ is at most $2\frac{\beta^2}{\beta - 1}\widetilde{T}^\star$. Case $4$ can add up to a factor of $\beta^{2\sqrt{\log_\beta(n)}}$ to the runtime and case $5$ can add up to $2\beta^{2\sqrt{\log_\beta(n)}}\widetilde{T}^\star$ to the runtime. It is possible that cases 4 and 5 will interchange. Since $\frac{1}{f}\geq 1$ and $\beta \geq 1$, the sequence with the worst runtime has Case 5 occurring maximally before Case 4 occurs. The final stage, which we ignored earlier, can add an additional factor of $\beta$. So,
\begin{align*}
    T &\leq 2\left(\frac{\beta^2}{\beta - 1} + 2\beta^{2\sqrt{\log_\beta(n)}}\right)\beta^{2\sqrt{\log_\beta(n)}}\beta \widetilde{T}^\star\\
    &\leq 2\left(\frac{\beta^2}{\beta - 1} + 2\right)\beta^{4\sqrt{\log_\beta(n)}}\beta \widetilde{T}^\star\\
    &\leq 4\frac{\beta^3}{\beta - 1}\beta^{4\sqrt{\log_\beta(n)}}\widetilde{T}^\star
\end{align*}

This completes the proof of the first claim of the theorem. Now, let us relate $\widetilde{T}^\star$
to $T^\star$ whose proof is given below.

\insertprethmspacing
\begin{lemma}\label{lemma:dynamic_prog_inequality}
$\widetilde{T}^\star = \Tcal_2\left(\{\bar{N}_i\}_{i=2}^n \}  \right) \leq (128\log((2/\omega)\log_2(192\Delta_{2}^{-2}/\omega)\lor 2))T^\star$
\end{lemma}

By applying Lemma~\ref{lemma:dynamic_prog_inequality}, we have that
\begin{align*}
    T &\leq 4\frac{\beta^3}{\beta - 1}\beta^{4\sqrt{\log_\beta(n)}}\widetilde{T}^\star\\
    &\leq 4\frac{\beta^{3+4\sqrt{\log_\beta(n)}}}{\beta - 1} (128\log((2/\omega)\log_2(192\Delta_{2}^{-2}/\omega))\lor 2)T^\star\\
    &\leq 512\frac{\beta^{3+4\sqrt{\log_\beta(n)}}}{\beta - 1} \left(\log((2/\omega)\log_2(192\Delta_{2}^{-2}/\omega))\lor 1/64\right)T^\star
\end{align*}
\qed

\textbf{Proof of Lemma~\ref{lemma:dynamic_prog_inequality}:}

In this proof alone, we will abuse notation by letting $\widetilde{r}$ and $\widetilde{n}_i$ be a feasible point for the dynamic program minimization problem and not an optimal solution that achieves $\widetilde{T}^\star$ runtime. By first applying the definition of $\widetilde{T}^\star$, we have
\begingroup
\allowdisplaybreaks
\begin{align*}
    \widetilde{T}^\star &= \min_{\widetilde{r}, \widetilde{n}_i}\sum_{i=1}^{\widetilde{r}} \lambda\left((n - \widetilde{n}_{i-1, \rm tot})(\bar{N}_{n - \widetilde{n}_{i, \rm tot} + 1} - \bar{N}_{n - \widetilde{n}_{i-1, \rm tot} + 1})\right)\\
    &= \min_{\widetilde{r}, \widetilde{n}_i}\lambda\left(n(\bar{N}_{n - \widetilde{n}_{1, \rm tot} + 1})\right) + \sum_{i=2}^{\widetilde{r}} \lambda\left((n - \widetilde{n}_{i-1, \rm tot})(\bar{N}_{n - \widetilde{n}_{i, \rm tot} + 1} - \bar{N}_{n - \widetilde{n}_{i-1, \rm tot} + 1})\right)\\
    &= \min_{\widetilde{r}, \widetilde{n}_i}\lambda\left(n \left(1 + 64\Delta_{n - \widetilde{n}_{1, \rm tot} + 1}^{-2}\log((2/\omega)\log_2(192\Delta_{n - \widetilde{n}_{1, \rm tot} + 1}^{-2}/\omega))\right)\right)\\
    &+ \sum_{i=2}^{\widetilde{r}} \lambda((n - \widetilde{n}_{i-1, \rm tot})(
    64\Delta_{n - \widetilde{n}_{i, \rm tot} + 1}^{-2}\log((2/\omega)\log_2(192\Delta_{n - \widetilde{n}_{i, \rm tot} + 1}^{-2}/\omega))\\
    &- 64\Delta_{n - \widetilde{n}_{i-1, \rm tot} + 1}^{-2}
    \log((2/\omega)\log_2(192\Delta_{n - \widetilde{n}_{i-1, \rm tot} + 1}^{-2}/\omega))))\\
    &\leq \lambda\left(n\right) + \min_{\widetilde{r}, \widetilde{n}_i}\lambda\left(64n\Delta_{n - \widetilde{n}_{1, \rm tot} + 1}^{-2}\log((2/\omega)\log_2(192\Delta_{2}^{-2}/\omega))\right)\\
    &+ \sum_{i=2}^{\widetilde{r}} \lambda(
    64\log((2/\omega)\log_2(192\Delta_{2}^{-2}/\omega))(n - \widetilde{n}_{i-1, \rm tot})(\Delta_{n - \widetilde{n}_{i, \rm tot} + 1}^{-2} - \Delta_{n - \widetilde{n}_{i-1, \rm tot} + 1}^{-2}))\\
    &= \lambda\left(n\right) + \min_{\widetilde{r}, \widetilde{n}_i}\sum_{i=1}^{\widetilde{r}} \lambda(
    64\log((2/\omega)\log_2(192\Delta_{2}^{-2}/\omega))(n - \widetilde{n}_{i-1, \rm tot})(\Delta_{n - \widetilde{n}_{i, \rm tot} + 1}^{-2} - \Delta_{n - \widetilde{n}_{i-1, \rm tot} + 1}^{-2}))\\
    &\leq \lambda\left(n\right) + (64\log((2/\omega)\log_2(192\Delta_{2}^{-2}/\omega))\lor 1)\min_{\widetilde{r}, \widetilde{n}_i}\sum_{i=1}^{\widetilde{r}} \lambda(
    (n - \widetilde{n}_{i-1, \rm tot})(\Delta_{n - \widetilde{n}_{i, \rm tot} + 1}^{-2} - \Delta_{n - \widetilde{n}_{i-1, \rm tot} + 1}^{-2}))\\
    &= \lambda(n) + (64\log((2/\omega)\log_2(192\Delta_{2}^{-2}/\omega))\lor 1) T^\star\\
    &\leq (1 + (64\log((2/\omega)\log_2(192\Delta_{2}^{-2}/\omega))\lor 1))T^\star\\
    &\leq (128\log((2/\omega)\log_2(192\Delta_{2}^{-2}/\omega))\lor 2)T^\star
\end{align*}
\endgroup
The second equation uses $\bar{N}_{n+1} = 0$ and $\widetilde{n}_0 = 0$. The third equation plugs in $\bar{N}_i$. The fourth equation applies Lemma~\ref{lemma:splitting} to the first term before pulling it out of the optimization problem and uses $\Delta_2 \leq \Delta_i$ in the remaining terms. The fifth equation rearranges terms by placing the second term back in the sum, using $\widetilde{n}_0 = 0$. The sixth equation uses Lemma~\ref{lemma:scaling_factor}. The seventh equation uses the definition of $T^\star$. The eighth equation uses the fact that each arm must be pulled at least once before elimination, so $\lambda(n) \leq T^\star$. The final equation bounds the sum by a factor of 2.
\qed

\insertprethmspacing
\begin{fact}\label{fact:subpoly}
The factor $\beta^{4\sqrt{\log_\beta(n)}}$ grows slower than $n^\alpha$ for any $\alpha > 0$. To view this, observe that
\begingroup
\allowdisplaybreaks
\begin{align*}
    \lim_{n\rightarrow \infty} \frac{\beta^{4\sqrt{\log_\beta(n)}}}{n^\alpha} &= \beta^{\lim_{n\rightarrow \infty} \log_\beta \left(\frac{\beta^{4\sqrt{\log_\beta(n)}}}{n^\alpha}\right)}\\
    &= \beta^{\lim_{n\rightarrow \infty} 4\sqrt{\log_\beta(n)} - \alpha\log_\beta(n)}\\
    &= \beta^{\lim_{n\rightarrow \infty} \log_\beta(n)\left(\frac{4}{\sqrt{\log_\beta(n)}} - \alpha\right)}\\
    &= \beta^{-\infty}\\
    &= 0
\end{align*}
\endgroup
By a similar proof, it can be shown that the function grows faster than any polylogarithmic function.
\end{fact}

\insertprethmspacing
\begin{remark}
APR can be easily modified for top-k arm identification as in~\cite{jun2016top} by eliminating arms when their UCB is below the k-th highest LCB. Arms are additionally thrown into an acceptance set (and not sampled anymore) if their LCB is higher than the $k+1$-th highest UCB. When $k$ arms have been accepted, the algorithm terminates. By redefining the $\Delta_i$ values in terms of the means of the $k$-th and $k+1$-th best arms, we can obtain an identical runtime bound with the same error probability.
\end{remark}

\section{Proof of the Lower Bound in the Fixed Confidence Setting}
\label{app:fclb}

In this section, we will prove the lower bound (Theorem~\ref{thm:fixed_confidence_lowerbound}) in
the fixed confidence setting.
Throughout this section, we will assume, without loss of generality, that the samples for each arm
$k\in[n]$ are generated upfront and each time we pull an arm, these observations are revealed in
order.

Our first result below shows that it is sufficient to consider algorithms which do not wait for
some time before executing arm pulls (thereby resulting in idle resources), unless it waits for
some arm pulls to complete at some point in the future (possibly to incorporate that information
before planning future pulls).
For this,
recall from Section~\ref{sec:lower_bound_fc} that $\Acaldelsf$ is the class of algorithms that
which can identify the best arm with probability at least $1-\delta$ for the given scaling
function $\scalefunc$ on any set of distributions $\theta\in\Theta$.
Let $\Acaltildedelsf$ denote the subset of algorithms in $\Acaldelsf$, where the algorithm
starts a new set of arm pulls at time $0$ or precisely at the same time
when a previous set of pulls are completed;
i.e for an algorithm in $\Acaltildedelsf$,
if $t_{s1} \leq t_{s2}\leq \dots$ denote the times at which new arm pulls are started,
and if $0<t_{e1} \leq t_{e2}\leq \dots$ denote the times at which previous pulls ended,
then, $t_{s1}=0$ and for all $i\geq 2$, $t_{si}=t_{ej}$ for some $t_{ej}$.
We have the following result.

\insertprethmspacing
\begin{lemma}
\label{lem:noidleresources}
Let $A\in\Acaldelsf$.
Then, there exists $\Atilde\in\Acaltildedelsf$ such that $\Atilde$ outputs the same arm as $A$
and $T(\Atilde,\theta) \leq T(A, \theta)$ a.s.
\end{lemma}
\begin{proof}
For the given $A\in\Acaldelsf$,
let $t'_1$ denote the first time instant at which some arm pulls are started in $A$ and when a
previous set of arm pulls were not completed.
Let $A_1$ be the algorithm that shifts the arm pulls at $t'_1$ forward to $t''_1$
defined as follows,
\[
t''_1=\max\{t\leq t'_1; \text{some pulls were completed at time $t$ in $A$}\}.
\]
As $A$ has the same information at $t''_1$ as it has at $t'_1$, and since the resources used for
the pulls at $t'_1$ are also available at $t''_1$,
 $A_1$ can execute the same decisions, and moreover by the assumptions at the beginning of
this section, it will
observe the same rewards.
Hence, it will output the same arm as $A$.
Moreover, since none of the arm
pulls in $A_1$ are started later than in $A$, it will not complete later than $A$.

The above transformation of the algorithm $A$, achieves the same output as $A$ and finishes no later
than $A$.
We can now keep repeating this transformation.
For example, $t'_2$ can denote the first time instant in which some arm pulls are started
in $A_1$ and when a
previous set of arm pulls were not completed.
Let $A_2$ be the resulting algorithm after shifting it forward as described above.
Since there are at most finitely many arm pulls in $A$, the above procedure should terminate in some
algorithm $\Atilde$, which achieves the same output as $A$ and finishes no later
than $A$.
Moreover, this algorithm, by definition, lies in $\Acaldelsf$.
%
%
%
\end{proof}

Lemma~\ref{lem:noidleresources} establishes that the algorithm which achieves the lowest expected time
is in $\Acaltildedelsf$, and therefore, we can restrict our attention to this subclass.
For such an algorithm, we can index the times at which decisions were made by $r\in\NN$
and let $\{t_r\}_{r\in\NN}$ denote the corresponding times so that
$0=t_0 < t_1 < t_2 < \dots $.
For $r\geq 0$,
let $\actr$ denote the action taken by the algorithm in round $r$, where each action is a
multiset of $(i, \alpha)$ pairs for each arm pull started by the algorithm at time $\timr$;
here, $i\in[n]$ is an arm that was chosen for evaluation and $\alpha\in[0,1]$ is the
fraction of the resource assigned to that arm pull.
Similarly, for $r\geq 1$, let $\fdbr$ denote the observations received at time $t_r$,
where $\fdbr$ is a multiset of $(i, Y)$ tuples with $i\in[n]$ being an arm that was chosen for
evaluation at some previous time and $Y$ is the observed reward from this pull.
An algorithm $A\in\Acaltildedelsf$ can be completely characterized by the action rule, the
stopping rule, and the recommendation rule, defined below.
\begin{enumerate}
\item The \emph{action rule} is a method to select new arms to evaluate and the amount of resources
to assign to each of them.
$\actr$ is $\filr$ measurable, where $\filr = \sigma(\{\fdbrr{r'}\}_{r'=1}^r)$ is the
sigma-algebra generated by past observations.
\item The \emph{stopping rule} $\rho$, is when the algorithm decides to stop executing more arm
pulls. It is a stopping time with respect to $\{\filr\}_{r\in \NN}$ satisfying
$\PP(\rho<\infty) = \PP(\timrr{\rho}<\infty) = 1$.
\item The \emph{recommendation rule} is the arm in $[n]$ chosen by the algorithm as the best arm.
It is a $\filrho$-measurable.
\end{enumerate}
Next, recall, our notation from Section~\ref{sec:lower_bound_fc}, where parenthesized subscripts
$\thetapk$ index the arms while subscripts without parentheses $\thetak$ are the distributions
ordered in decreasing order of mean value.
Similarly let $\mupk,\muk$ denote the means of $\thetapk,\thetak$ respectively and recall that
$\muone<\mutwo\leq \muthree \dots \leq \mun$.
For $r\in\NN$ and $k\in[n]$, we will let $\Nkr$ denote the number of completed
arm pulls for arm $(k)$ from time $0$ to time $t_r$.
Let $\Nr=\sum_{k\in[n]}\Nkr$ denote the number of completed arm pulls of all the arms 
 from time $0$ to time $t_r$.
Finally, we may define the gaps for a set of $n$ distributions $\theta\in\Theta$
as follows.
This is consistent with our previous definition~\eqref{eqn:Deltai}.
\begin{align*}
\Deltapk = \muone - \mutwo \quad\text{if}\;\; \thetapk = \thetaone, 
\hspace{1.1in}
\Deltapk = \muone - \mupk \quad\text{otherwise. }
\end{align*}

Our proof of the lower bound first establishes lower bounds on the sample complexity for BAI and
then translates these lower bounds to a lower bound on the runtime for an algorithm in our setting.
We will prove this hardness result for problems in $\Theta$ with Gaussian reward distributions,
although it is straightforward to generalise it to distributions for which the KL divergence
satisfies
${\rm KL}(\nu \| \nu') \asymp (\mu - \mu')^2$ where $\mu,\mu'$ are respectively the means of
$\nu,\nu'$.
The following result is an adaptation of a well-known hardness result for BAI to algorithms in
in $\Acaltildedelsf$.
We have given its proof in Appendix~\ref{app:samplecomplexity}.

\insertprethmspacing
\begin{lemma}
\label{lem:samplecomplexity}
Let $\theta\in\Theta$ such that the reward distribution $\thetapk$ is Gaussian with unit variance
for all $k\in[n]$.
Then, any $A\in\Acaltildedelsf$ satisfies,
\[
\EE[\Nkrho] \geq \frac{2}{\Deltapk^2} \log\left( \frac{1}{2.4\delta} \right).
\]
\end{lemma}

\vspace{0.1in}

We will prove the lower bound for scaling functions $\scalefunc$ which satisfy,
for some $\alpha_1, \alpha_2 > 0$,
\begin{align*}
\alpha_1 m \leq \scalefunc(m) \leq \alpha_2 m
\hspace{0.2in}
\text{for all } m\geq 1.
\numberthis \label{eqn:scalfuncasm}
\end{align*}
If $\scalefunc(n) = \beta_1 m +  \beta_2 \scalefunc'(m)$, where
$\scalefunc'(m)\in\littleO(m)$ and is concave with $\scalefunc'(0) = 0$, then $\scalefunc$ satisfies all the properties
in Section~\ref{sec:fixedconf} and satisfies the above conditions for
$\alpha_1 = \beta_1$ and $\alpha_2 = \beta_1 + \beta_2 \scalefunc'(1)$.
Note that since we are assuming $\scalefunc$ is increasing and concave, 
the upper bound in~\eqref{eqn:scalfuncasm} is already implied  by Lemma~\ref{lemma:scaling_factor};
in particular, $\scalefunc(m) \leq \scalefunc(1) m$ for all $m\geq 1$.
While the $\scalefunc(m) \geq \alpha_1 m$ condition means that the proposed algorithm is
optimal only for a limited class of scaling functions, it is worth noting that it also captures
a practically meaningful use case.
To see this, note that the above conditions imply
$\lim_{m\rightarrow\infty} \frac{\scalefunc(m)}{m} \geq \alpha_1$.
This states that even when the resource is used most efficiently as possible (i.e. at the maximum
possible throughput), a minimum amount of
`work' needs to be done to execute one arm pull.

We shall now prove the lower bound for scaling functions which satisfy~\eqref{eqn:scalfuncasm} and
for reward distributions which satisfy the conditions in Lemma~\ref{lem:samplecomplexity}.

\subsection{Proof of Theorem~\ref{thm:fixed_confidence_lowerbound}}

In this proof,
without loss of generality, we will assume that the arms are ordered so that
$\mupone > \muptwo \geq \mupthree \geq \dots \mupn$.
Let $A\in\Acaltildedelsf$ be arbitrary and
let $\rho$ denote the round at which the algorithm $A$ decided to stop and recommend an arm.
The minimum time taken to execute these arm pulls can be upper bound
as follows.
\[
T(A, \theta) \geq \scalefunc\left( \sum_{k=1}^n \Nkrho \right)
\geq \alpha_1 \left(\sum_{k=1}^n \Nkrho \right).
\]
Here, the first step simply assumes that all arms were pulled exactly $\Nkrho$ times in
the first set of pulls and the second uses~\eqref{eqn:scalfuncasm}.
Now define $\Ntildekr$ recursively as follows for all $r\in \NN$,
\[
\Ntildekkrr{1}{r} = \Nkkrr{1}{r},
\hspace{0.4in}
\Ntildekr = \min\left(\Ntildekkrr{k-1}{r}, \Nkr\right)
\text{ for } k\geq 2.
\]
Next, denote $\Kcal(A) = \{2\} \cup \{3\leq k \leq n;
\Ntildekrho = \Nkrho\} \cup \{n+1\}$.
Write $\Kcal(A) = \{k_1, k_2, \dots, \kptilde, \kptildepo\}$ so that the $k_i$'s are in
ascending order, i.e. 
$2=k_1 < k_2 < \dots < \kptilde <\kptildepo = n+1$.
Observe that, by definition, we have $\Nkkrr{k}{\rho} \geq \Nkkrr{k_i}{\rho}$ for
$k=k_i,\dots, k_{i+1}-1$.
We can now upper bound $\EE[T(A,\theta)]$ as follows.
\begingroup
\allowdisplaybreaks
\begin{align*}
\EE[T(A, \theta)]
&\geq \alpha_1 \left(\EE\Nkkrr{1}{\rho} +
            \sum_{k=k_1}^{k_2-1} \EE\Nkrho + \sum_{k=k_2}^{k_3-1}\EE\Nkrho + \dots
                      \sum_{k=\kptilde}^{\kptildepo-1} \EE\Nkrho  \right)
\\
&\geq \alpha_1 \left( \EE\Nkkrr{1}{\rho} + (k_2-k_1)\EE\Nkkrr{2}{\rho} + 
(k_3-k_2)\EE\Nkkrr{k_2}{\rho} +
\dots +
                        (\kptildepo - \kptilde)\EE\Nkkrr{\kptilde}{\rho}  \right)
\\
&\geq 2\alpha_1 \log\bigg(\frac{1}{2.4\delta} \bigg)
\bigg( (k_2-1)\frac{1}{\Deltaii{2}^2} +  (k_3-k_2)\frac{1}{\Deltaii{k_2}^2}
        + (k_4-k_3)\frac{1}{\Deltaii{k_3}^2} + \dots
                        (\kptildepo - \kptilde)\frac{1}{\Deltaii{\kptilde}^2} \bigg)
\label{eqn:tathetamanip}\numberthis
\\
&\geq 2\alpha_1 \log\bigg(\frac{1}{2.4\delta} \bigg)
\Bigg( (k_2-1)\bigg(\frac{1}{\Deltaii{2}^2} - \frac{1}{\Deltaii{k_2}^2}\bigg)
    + (k_3-1)\bigg(\frac{1}{\Deltaii{k_2}^2} - \frac{1}{\Deltaii{k_3}^2}\bigg)
    + \dots
\\
&\hspace{1.5in}
    + (\kptilde-1)\bigg(\frac{1}{\Deltaii{\kptildemo}^2} - \frac{1}{\Deltaii{\kptilde}^2}\bigg)
    + n\frac{1}{\Deltaii{\kptilde}^2} 
\Bigg)
\\
&\geq \frac{2\alpha_1}{\alpha_2} \log\left(\frac{1}{2.4\delta} \right)
\Bigg( \scalefunc\bigg((k_2-1)\bigg(\frac{1}{\Deltaii{2}^2} - \frac{1}{\Deltaii{k_2}^2}\bigg) \bigg)
    + \scalefunc\bigg((k_3-1)\bigg(\frac{1}{\Deltaii{k_2}^2} - \frac{1}{\Deltaii{k_3}^2}\bigg) \bigg)
    + \dots
\\
&\hspace{1.5in}
    + \scalefunc\bigg((\kptilde-1)\bigg(\frac{1}{\Deltaii{\kptildemo}^2} - \frac{1}{\Deltaii{\kptilde}^2}\bigg) \bigg)
    + \scalefunc\bigg(n\frac{1}{\Deltaii{\kptilde}^2} \bigg)
\Bigg)
\end{align*}
\endgroup
Here, the second step uses the definition of $\Kcal(A) = \{k_1, k_2, \dots, \kptilde\}$.
In the third step, first
we have used Lemma~\ref{lem:samplecomplexity}, while observing that $\Deltaone=\Deltatwo$;
then, we have used the fact that $k_1=2$, and therefore $k_2-k_1+1 = k_2-1$.
The third step can be obtained via the following series of manipulations.
First observe that we can write
\[
(k_2-1)\frac{1}{\Deltaii{2}^2} +  (k_3-k_2)\frac{1}{\Deltaii{k_2}^2}
= 
(k_2-1)\bigg(\frac{1}{\Deltaii{2}^2} - \frac{1}{\Deltaii{k_2}^2}\bigg)
    + (k_3-1)\frac{1}{\Deltaii{k_2}^2}.
\]
We then apply the same manipulation to $(k_3-1)\frac{1}{\Deltaii{k_2}^2}$ and the next term
in the summation in~\eqref{eqn:tathetamanip}.
Proceeding in this fashion, the coefficient for the last term will be
$\kptildepo - \kptilde + (\kptilde - 1) = (n+1) - 1 = n$.
Finally,  the last step uses~\eqref{eqn:scalfuncasm}.

The above bound is a lower bound on the run time of algorithm $A$  in terms of $\Kcal(A)$.
To lower bound the run time of any algorithm, it is sufficient to lower bound the above expression
over all values of $(k_2, \dots, \kptilde)$, where $\tilde{p}\leq n-1$
and $3\leq k_2<\dots \kptilde\leq  n$.
Writing $\alpha_1/\alpha_2 = c_\lambda$, we have,
\begin{align*}
\inf_{A\in\Acaltildedelsf} \EE[T(A, \theta)]
&\geq
2c_\lambda \log\left(\frac{1}{2.4\delta} \right)
\min_{(k_2, \dots, \kptilde)}
\Bigg( \scalefunc\bigg((k_2-1)\bigg(\frac{1}{\Deltaii{2}^2} - \frac{1}{\Deltaii{k_2}^2}\bigg) \bigg)
    + \scalefunc\bigg((k_3-1)\bigg(\frac{1}{\Deltaii{k_2}^2} - \frac{1}{\Deltaii{k_3}^2}\bigg) \bigg)
    + 
\\
&\hspace{2.0in}
   \dots
    + \scalefunc\bigg((\kptilde-1)\bigg(\frac{1}{\Deltaii{\kptildemo}^2} - \frac{1}{\Deltaii{\kptilde}^2}\bigg) \bigg)
    + \scalefunc\bigg(n\frac{1}{\Deltaii{\kptilde}^2} \bigg)
\Bigg)
\\
&=
2c_\lambda \log\left(\frac{1}{2.4\delta} \right) \Tstar.
\end{align*}
Here, we have observed that the expression in the RHS inside the minimum is simply the
expansion of the dynamic program $\Tcal_2\left(\{ \Deltai^{-2}\}_{i=2}^n\right)$,
see~\eqref{eqn:dp} and~\eqref{eqn:Tstar}.
The proof is completed by Lemma~\ref{lem:noidleresources}, which implies that
$\inf_{A\in\Acaltildedelsf} \EE[T(A, \theta)] = \inf_{A\in\Acaldelsf} \EE[T(A, \theta)]$.
\qedwhite

\subsection{Proof of Lemma~\ref{lem:samplecomplexity}}
\label{app:samplecomplexity}

Our proof of Lemma~\ref{lem:samplecomplexity} uses essentially the same argument
as~\citet{kaufmann2016complexity}.
However, due to the differences in our set up, we need to verify that the same proof carries
through.
For this, consider two sets of distributions $\theta,\theta'\in\Theta$ such that for all
$k\in[n]$, $\thetapk$ and $\thetappk$ are absolutely continuous with respect to each other.
This means, for each $k\in[n]$, there exists a measure such that $\thetapk$ and $\thetappk$ have
densities $\fk,\fpk$ with respect to this measure.
Let $\{(K_s, Z_s)\}_{s=1}^N$ be a sequence of arm-observation pairs received in order when executing
algorithm $A\in\Acaltildedelsf$.
Define the log-likelihood ratio $L_N$ after the first $N$ observations as shown below.
Let $\Ldag_r$ denote the log likelihood ratio using the observations up to round $r$,
i.e the first $\Nr$ observations.
We have:
\begin{align*}
\numberthis \label{eqn:llratio}
L_N
= L_N\left( \{(K_s, Z_s)\}_{s=1}^{N} \right)
= \sum_{k=1}^n \sum_{s=1}^{N} \indfone(K_s = k)\log\left(\frac{\fk(Z_s)}{\fpk(Z_s)} \right),
\hspace{0.3in}
\Ldag_r
= L_{\Nr}\left(\{(K_s, Z_s)\}_{s=1}^{\Nr} \right).
\end{align*}
The following result, from~\citet{kaufmann2016complexity} upper bounds the log likelihood
ratio at the stopping time.

\insertprethmspacing
\begin{lemma}[Lemma 19~\citep{kaufmann2016complexity}]
\label{lem:kaufman19}
Let $\theta, \theta'\in\Theta$ be sets of distributions such that
$\thetapk$ and $\thetappk$ are absolutely continuous with respect to each other for all $k\in[n]$.
Let $A\in\Acaltildedelsf$ and $\rho$ be any almost surely finite stopping time with respect to
$\{\filr\}_{r\in\NN}$.
Let $\Ldag_r$ be as defined in~\eqref{eqn:llratio}.
For every event $E\in\filrho$ (i.e. $E$ such that $E\cap(\rho=r) \in \filr$), 
\[
\EE_\theta[\Ldag_\rho] \geq d(\PP_\theta(E), \PP_{\theta'}(E)).
\]
Here, $d(x,y) = x\log(x/y) + (1-x)\log((1-x)/(1-y))$.
\end{lemma}
\begin{proof}
The first challenge in adapting the above result from~\citet{kaufmann2016complexity} for
the sequential setting to ours is to show the following technical result: 
\emph{
Let $Z_s, L_N$ be as defined above.
Then, for all $N\geq 0$, and for all measurable $g:\RR^N\rightarrow\RR$,
$
\EE_{\theta'}\left[g\left(\{Z_s\}_{s=1}^N\right)\right]
= \EE_{\theta}\left[g\left(\{Z_s\}_{s=1}^N\right)\exp\left(-L_N \right)\right]
$.
}
This claim can be proved using the same induction argument in Lemma 18
of~\citet{kaufmann2016complexity},
where the base case uses the fact that the algorithm's initial action does not depend on any
observations and the induction step notes that
$\EE_{\theta'}\left[g\left(\{Z_s\}_{s=1}^{N+1}\right) | \left(\{(K_s, Z_s)\}_{s=1}^N\right)\right]$
is a measurable function mapping $\RR^N$ to $\RR$.

Now, let $\filpN=\sigma(\{(K_s, Z_s)\}_{s=1}^N))$.
The above result implies, in particular, that for all $N\geq 1$
and any event $E'\in\filpN$, we have $\EE_\thetap[\indfone_{E'}]
=\EE_{\theta}[\indfone_{E'} \exp(-L_N)]$.
Now, let $E$ and $\rho$ be as given in the theorem statement.
\[
\PP_\thetap(E) =
\EE_\thetap[\indfone_E]
=\sum_{r=1}^\infty  \EE_\thetap[\underbrace{\indfone(E \cap \{\rho=r\})}_{\in \filr = \filprr{\Nr}}]
=\sum_{r=1}^\infty  \EE_\thetap[\indfone(E \cap \{\rho=r\})\exp(-L_{\Nr})]
=\EE_\theta[\indfone_E \exp(-\Ldag_\rho)].
\]
Given this result, the remainder of the proof is identical to the proof of
their Lemma 19.
\end{proof}

We are now ready to prove Lemma~\ref{lem:samplecomplexity}.

\begin{proof}[\textbf{Proof of Lemma~\ref{lem:samplecomplexity}}]
Recall from the beginning of this section that we assume the samples for each arm are generated
upfront and then revealed in order; let $\{\Yks\}_{s\geq 1}$ denote the sequence of these
observations for arm $k$.
We can write the log-likelihood ration $\Ldag_r$ as follows,
\[
\Ldag_r = \sum_{k=1}^n \sum_{s=1}^{\Nkr}
    \log\left(\frac{\fk(\Yks)}{\fpk(\Yks)} \right).
\]
We can now conclude,
for any event $E\in \filrho$,
\begin{align*}
\numberthis \label{eqn:kllb}
d(\PP_\theta(E), \PP_\thetap(E))
&\leq \EE_\theta[\Ldag_\rho]
= \EE_\theta\left[ \sum_{k=1}^n \sum_{s=1}^{\Nkr}
    \log\left(\frac{\fk(\Yks)}{\fpk(\Yks)} \right) \right]
=  \sum_{k=1}^n \EE_\theta[\Nkr] {\rm KL}(\thetak \| \thetappk)
\end{align*}
Here,
the first step uses Lemma~\ref{lem:kaufman19},
the second step uses the expression for $\Ldag_\rho$ above,
and the last step uses Wald's identity.

Let $k\in[n]$ be given and assume $\thetapk$ is not the best arm, i.e. $\mupk \neq \muone$;
we will handle the case where $\thetapk$ is the best arm later.
Now fix $\alpha > 0$ and
let $\thetap$ be an alternative model where $\thetappj=\thetaj$ for $j\neq k$ and
$\thetappk = \Ncal(\muone + \alpha, 1)$.
Let $E$ denote the event that the recommended arm is $k$.
Since $A\in\Acaldelsf$, we have $\PP_\thetap(E) \geq 1-\delta$ and $\PP_\theta(E) \leq 1-\delta$.
Therefore,
\[
\log\left(\frac{1}{2.4\delta}\right)
\leq d(\delta, 1-\delta)
\leq d(\PP_\theta(E), \PP_\thetap(E))
\leq  \sum_{j=1}^n \EE_\theta[\Njr] {\rm KL}(\thetaj \| \thetappj)
= \EE_\theta[\Nkr] \frac{(\Deltapk + \alpha)^2}{2}.
\]
Here, the first step uses the algebraic identity $d(\delta, 1-\delta) \geq \log(1/(2.4\delta))$,
the second step uses the fact that $d(x,y)$ is increasing in $x$ for $x>y$ and decreasing when
$x<y$,
and the third step uses~\eqref{eqn:kllb}.
The last step first observes that the KL divergence between all arms is zero except between
$\thetapk$ and $\thetappk$;
moreover, as the KL divergence between two Gaussians can be expressed as
${\rm KL}(\Ncal(\mu_1, 1) \|\Ncal(\mu_2, 1) ) = (\mu_1 - \mu_2)^2/2$,
we
have ${\rm KL}(\thetak \| \thetappk) = (\muone + \alpha - \mupk)^2/2 = (\Deltapk+\alpha)^2/2$.
Therefore,
$\EE[\Nkrho] \geq \frac{2}{(\Deltapk+\alpha)^2} \log\left( \frac{1}{2.4\delta} \right)$.
The statement is true for all $\alpha>0$, so letting $\alpha\rightarrow 0$ yields the claim
for all $\thetapk$ that are not the best arm.

The proof for when $\thetapk$ is the best arm can be obtained by considering a model $\thetap$
where $\thetappj = \thetapj$ for $j\neq k$
and $\thetappk = \Ncal(\mutwo - \alpha, 1)$ and following a similar argument.
\end{proof}

\section{Proof of Upper Bound in the Fixed Deadline Setting}

\subsection{Proof of Theorem~\ref{thm:fixed_budget_error_prob}}
\label{app:fixed_budget}


In this section, we will prove Theorem~\ref{thm:fixed_budget_error_prob}. The analysis will borrow techniques from the proof for sequential halving~\citet{karnin2013almost}. We will utilize the following lemmas in the proof. Let $\hat{p}_i^r$ be the empirical mean of arm $i$ at round $r$.
The following lemma is a direct application of
Hoeffding's inequality on the difference of empirical means $\hat{p}_1^r - \hat{p}_i^r$.

\insertprethmspacing
\begin{lemma}\label{lemma:Hoeffding}
Assume that the best arm was not eliminated prior to round $r$. Then, for any arm $i\in S_r$,
\begin{align*}
    \mathrm{Pr}[\hat{p}_1^r < \hat{p}_i^r] \leq \exp \left(-\frac{1}{2}t_r \Delta_i^2 \right)
\end{align*}
\end{lemma}

The following lemma is a version of Lemma 4.3 from~\citet{karnin2013almost}. We will use this lemma with the union bound to bound the probability of eliminating an arm at a given stage.

\begin{lemma}
Suppose $S$ arms remain and each arm was pulled $t$ times. The probability that $\left\lceil\frac{|S|}{2}\right\rceil$ of the arms have greater empirical mean than the best arm is at most:
\begin{align*}
    3\exp \left(-\frac{1}{2}t\Delta_{i_{S}}^2\right)
\end{align*}
where $i_{S} = \left\lceil \frac{1}{2} \left\lceil \frac{|S|}{2} \right\rceil \right\rceil$
\label{lemma:substage_bound}
\end{lemma}
\begin{proof}
Define $S'$ as the set of arms in $S$ minus the arms with the top $\left\lceil \frac{1}{2} \left\lceil \frac{|S|}{2} \right\rceil \right\rceil$ true means. In order for half of the arms to have higher empirical mean than the best arm, at least $\left\lceil \frac{|S|}{2} \right\rceil - \left\lceil \frac{1}{2} \left\lceil \frac{|S|}{2} \right\rceil \right\rceil + 1 = \left\lceil \frac{|S|}{2} \right\rceil - \left\lceil \frac{1}{2} \left\lfloor \frac{|S|}{2} \right\rceil \right\rceil \geq  \frac{|S|}{4}$ arms in $S'$ must have greater empirical mean than the best arm.
Since $S'$ has size $|S| - \left\lceil \frac{1}{2} \left\lceil \frac{|S|}{2} \right\rceil \right\rceil \leq \frac{3}{4}|S|$, at least $\frac{1}{3}$-rd of the arms in $S'$ need to have higher empirical mean than the best arm's empirical mean. Let us first upper bound $N$, the number of arms in $S'$ that have higher empirical mean than the top arm.
\begingroup
\allowdisplaybreaks
\begin{align*}
    \mathbb{E}\left[N\right] &= \sum_{i \in S'} \mathrm{Pr}[\hat{p}_1 < \hat{p}_i]\\
    &\leq \sum_{i \in S'} \exp \left(-\frac{1}{2}t\Delta_i^2 \right)\\
    &\leq |S'|\max_{i \in S'} \exp \left(-\frac{1}{2}t\Delta_i^2\right)\\
    &\leq |S'|\exp \left(-\frac{1}{2}t\Delta_{i_{S}}^2\right)\\
\end{align*}
\endgroup

Now, let's use Markov's inequality to bound the desired probability:
\begin{align*}
    \textrm{Pr}\left[N > \frac{1}{3}S'\right] &\leq 3\frac{\mathbb{E}[N]}{S'}\\
    &\leq 3\exp \left(-\frac{1}{2}t\Delta_{i_{S}}^2\right)
\end{align*}

\end{proof}

In the following lemma, we will lower bound the number of pulls each surviving arm receives, and show that it increases by a factor of $2^{k}$ each stage.
\begin{lemma}
The number of pulls per arm at stage $r$ of Algorithm~\ref{alg:fixed_budget}, $t_r$, is greater than:
\begin{align*}
    2^{rk}(2^{k} - 1) x(k)
\end{align*}
\label{lemma:pull_bound}
\end{lemma}

\begin{proof}
We will show this by induction. Let $r=0$.

\begin{align*}
    (2^k-1)x(k) &= (2^k-1)\left\lfloor \frac{ \lambda^{-1}\left(\frac{T}{ \lceil \log_{2^{k}}(n) \rceil}\right)}{2^{k \left\lceil\log_{2^k}(n)\right\rceil}(2^k - 1)} \right \rfloor\\
    &\leq (2^k-1)\left \lfloor \frac{\lambda^{-1}\left(\frac{T}{ \lceil \log_{2^{k}}(n) \rceil}\right)}{n(2^k - 1) } \right \rfloor\\
    &\leq \left \lfloor \frac{\lambda^{-1}\left(\frac{T}{ \lceil \log_{2^{k}}(n) \rceil}\right)}{n } \right \rfloor\\
    &= t_0
\end{align*}

Suppose this lemma holds for some $r$. Then,
\begingroup
\allowdisplaybreaks
\begin{align*}
    2^{(r+1)k}(2^{k} - 1) x(k) &=  2^{(r+1)k}(2^{k} - 1)\left\lfloor \frac{\lambda^{-1}\left(\frac{T}{ \lceil \log_{2^{k}}(n) \rceil}\right)}{2^{k \left\lceil\log_{2^k}(n)\right\rceil}(2^k - 1) } \right \rfloor\\
    &\leq  2^{(r+1)k}\left\lfloor \frac{\lambda^{-1}\left(\frac{T}{ \lceil \log_{2^{k}}(n) \rceil}\right)}{2^{k \left\lceil\log_{2^k}(n)\right\rceil}} \right \rfloor\\
    &\leq  \left\lfloor \frac{2^{(r+1)k}\lambda^{-1}\left(\frac{T}{ \lceil \log_{2^{k}}(n) \rceil}\right)}{2^{k \left\lceil\log_{2^k}(n)\right\rceil}} \right \rfloor
\end{align*}
\endgroup

Suppose $n$ is a factor of $2^k$. Then $|S_r| = \frac{n}{2^{rk}}$.
\begin{align*}
    \left\lfloor \frac{2^{(r+1)k} \lambda^{-1}\left(\frac{T}{ \lceil \log_{2^{k}}(n) \rceil}\right)}{2^{k \left\lceil\log_{2^k}(n)\right\rceil}} \right \rfloor
    &= \left\lfloor \frac{2^{(r+1)k}\lambda^{-1}\left(\frac{T}{ \lceil \log_{2^{k}}(n) \rceil}\right)}{n} \right \rfloor\\
    &= \left\lfloor \frac{\lambda^{-1}\left(\frac{T}{ \lceil \log_{2^{k}}(n) \rceil}\right)}{|S_{r+1}|} \right \rfloor\\
    &= t_r
\end{align*}

Suppose $n$ is not a factor of $2^k$. Then, $|S_r| \leq \frac{n}{2^{(r-1)k}}$.
\begin{align*}
    \left\lfloor \frac{2^{(r+1)k} \lambda^{-1}\left(\frac{T}{ \lceil \log_{2^{k}}(n) \rceil}\right)}{2^{k \left\lceil\log_{2^k}(n)\right\rceil}} \right \rfloor
    &\leq \left\lfloor \frac{2^{rk} \lambda^{-1}\left(\frac{T}{ \lceil \log_{2^{k}}(n) \rceil}\right)}{n} \right \rfloor\\
    &\leq \left\lfloor \frac{\lambda^{-1}\left(\frac{T}{ \lceil \log_{2^{k}}(n) \rceil}\right)}{|S_{r+1}|} \right \rfloor\\
    &= t_r
\end{align*}

\end{proof}

Now, we will use the previous lemmas to prove a lemma that bounds the probability of elimination in a given stage.

\begin{lemma}\label{lemma:round_survival}
Assume the best arm has not been eliminated by round $r$. If $|S_r|\geq 2^k$, the probability that the best arm is eliminated in round $r$ is at most:

\begin{align*}
    3k \exp \left(-\frac{nx(k)}{8H_2}\right)
\end{align*}

Otherwise, the probability that the best arm is eliminated is at most

\begin{align*}
    3\lceil \log_2(|S_r|)\rceil \exp \left(-\frac{nx(k)}{8H_2}\right)
\end{align*}

\end{lemma}
\begin{proof}
Let's assume $|S_r|\geq 2^k$. In round $r$, each arm is sampled $t_r \geq 2^{rk}(2^{k} - 1) x(k)$ times. We only keep the top $\frac{1}{2^k}$ (by their empirical means) fraction of the arms in $S_r$. 
Suppose we eliminate arms in the following fashion: first we eliminate the lowest $\left\lfloor \frac{|S_r|}{2} \right\rfloor$ arms, then the next $\left\lfloor \frac{|S_r|}{4} \right\rfloor$, and so on ($k$ times). Call the event that the best arm is eliminated in the $i$-th such elimination procedure $B_i$. Call the set of arms that survived until the $i$-th event $S_{r,i}$ We thus want to upper bound:
\begingroup
\allowdisplaybreaks
\begin{align*}
    P\left(\bigcup_{j=1}^k B_j\right) &\leq \sum_{j=1}^k P(B_j)\\
    &\leq 3 \sum_{j=1}^k \exp \left(-\frac{1}{2}t_r\Delta_{i_{S_{r,j}}}^2\right)\\
    &\leq 3 \sum_{j=1}^k \exp \left(-\frac{1}{2}2^{rk}(2^k - 1)x(k)\Delta_{i_{S_{r,j}}}^2\right)\\
    &\leq 3 \sum_{j=1}^k \exp \left(-\frac{1}{2}2^{rk + j-1}x(k)\Delta_{i_{S_{r,j}}}^2\right)
\end{align*}
\endgroup
The first line follows from the union bound. The second line follows via application of Lemma~\ref{lemma:substage_bound}. The third step follows since $t_r \geq 2^{rk}(2^{k} - 1) x(k)$ by Lemma~\ref{lemma:pull_bound}. The last step applies $2^{j-1} \leq 2^k - 1$.

Observe that $i_{S_{r,i}} \geq \frac{n}{2^{rk + i + 1}}$. So,
\begin{align*}
    3 \sum_{j=1}^k \exp \left(-\frac{1}{2}2^{rk + j-1}X_k\Delta_{i_{S_{r,j}}}^2\right)
    &\leq 3 \sum_{j=1}^k \exp \left(-\frac{1}{8}2^{rk+j+1}x(k)\Delta_{i_{S_{r,j}}}^2\right)\\
    &\leq 3 \sum_{j=1}^k \exp \left(-\frac{n}{8}\frac{\Delta_{i_{S_{r,j}}}^2}{i_{S_{r,i}}}\right)\\
    &\leq 3k \exp \left(-\frac{nx(k)}{8H_2}\right)
\end{align*}
The second line uses the bound on $i_{S_{r,i}}$, and the third line maximizes the expression in the sum over $i$ (recall that $H_2 = \max_{i\neq 1} i\Delta^{-2}_i$).

The proof of the second half of the lemma follows by replacing $k$ with $\lceil \log_2(|S|)\rceil$ in the previous steps.

\end{proof}

We are now ready to prove Theorem~\ref{thm:fixed_budget_error_prob}.

\textbf{Proof of Theorem~\ref{thm:fixed_budget_error_prob}:}
\begin{proof}

The algorithm always terminates in $T$ time, as it takes $\left\lceil\log_{2^k}(n)\right\rceil$ stages that each take at most $T/\left\lceil\log_{2^k}(n)\right\rceil$ time. So, let's upper bound the probability of error.

To select the best arm, it must survive all $\left\lceil \frac{\log_2(n)}{k} \right\rceil$ rounds, after which it will be the last arm remaining. Using the union bound with Lemma~\ref{lemma:round_survival}, we can upper bound the probability of elimination of the best arm by:

\begin{align*}
    &\sum_{r=0}^{\lfloor\log_2(n)/k\rfloor} 3k \exp \left(-\frac{nx(k)}{8H_2}\right) + 3\lceil \log_2(|S_{r_f - 1}|)\rceil \exp \left(-\frac{nx(k)}{8H_2}\right)\\
    &= 3\lceil\log_2(n)\rceil \exp \left(-\frac{nx(k)}{8H_2}\right)\label{eq:budget_alg_bound}\numberthis
\end{align*}

\end{proof}

\begin{remark}
Applying Theorem~\ref{thm:fixed_budget_error_prob} with $k=1$, and $\lambda(m) = m^{-1}$, assuming $n$ is a power of $2$ (for simplicity), and that $1$ is evenly divisible by $\frac{n\log_2(n)}{T}$, we recover the error probability upper bound of sequential halving~\cite{karnin2013almost}:
\begin{align*}
    3\log_2(n)\exp\left(-\frac{nT}{8H_2\log_2(n)}\right)
\end{align*}
\end{remark}

\begin{remark}
The naive extension of sequential halving (which is Algorithm~\ref{alg:fixed_budget} with $k=1$) has a higher bound on the error probability than Algorithm~\ref{alg:fixed_budget} with $k=\kstar$. This is because $x(\kstar) \geq x(1)$ and the error probability upper bound (Equation~\ref{eq:budget_alg_bound}) is decreasing in $x(k)$.\label{remark:scsh_vs_sh}
\end{remark}








\subsection{Lower bounding $x(k)$ for $\speedfunc(m)=m^{q}$}
\label{app:fbexample}

In this section, we will lower bound $x(k)$, which will allow us to upper bound the error probability of Algorithm~\ref{alg:fixed_budget} when contextualizing the results for different values of $k$.

Assuming $x(k) \geq 1$ and $2^k \leq n$:
\begingroup
\allowdisplaybreaks
\begin{align*}
    x(k) &= \left \lfloor \frac{\lambda^{-1}\left(\frac{T}{ \lceil \log_{2^{k}}(n) \rceil}\right)}{2^{k \left\lceil\log_{2^k}(n)\right\rceil}(2^k - 1)} \right \rfloor\\
    &=  \frac{ \lambda^{-1}\left(\frac{T}{ \lceil \log_{2^{k}}(n) \rceil}\right)}{2^{k}n(2^k - 1)}\\
    &\geq \frac{\lambda^{-1}\left(\frac{T}{ \lceil \log_{2^{k}}(n) \rceil}\right)}{4^{k}n}\\
    &\geq \frac{1}{4^{k}n}\left(\frac{T}{ \lceil \log_{2^{k}}(n) \rceil}\right)^{1/q}\\
    &\geq \frac{1}{4^{k}n}\left(\frac{Tk}{ 2\log_{2}(n)}\right)^{1/q}
\end{align*}
\endgroup
In the second line, we upper bound the first term in the denominator, which may be greater by a factor of $2^k$. In the fourth line, we plug in $\lambda^{-1}(t) = t^{1/q}$. In the last step, we upper bound the ceil using $2^k \leq n$.

\section{Additional experiments: Fixed deadline setting:}
In this section, we present additional experiments evaluating \scshs against baselines on a set of simulation experiments and on cosmological parameter estimation task.
\insertFigFixedDeadline

\subsection{Baselines:}
In the fixed deadline setting, we compare \scshs to SH~\citep{karnin2013almost} and
the UCB-E algorithm~\citep{audibert2010best}.

\subsection{Simulation experiments}
We evaluate Algorithm~\ref{alg:fixed_budget} on a synthetic domain consisting of $n = 1024$
Bernoulli arms with means sampled uniformly from $[0, 1]$. We use a scaling function
of the form $\lambda(m) = m^{q}$ for 
different choices of $q \in \{0.1, 0.25, 0.5, 0.9\}$.
All settings are evaluated ten times, and we report the mean and
standard error. We evaluate the accuracy of algorithms for
varying time budgets. In Figure~\ref{fig:fixed_budget_exps}, we observe that
Algorithm~\ref{alg:fixed_budget} (with $k=\kstar$) consistently matches or outperforms SH
and UCB-E.
This difference is more pronounced as the scaling is poor (smaller $q$).

\subsection{Physical experiment: Estimation of cosmological parameters}\label{app:cosmological_estimation}
We evaluate Algorithm~\ref{alg:fixed_budget} on a variant of a problem in~\citet{kandasamy2019dragonfly}, where the goal is to estimate the Hubble constant, dark matter fraction, and dark energy fraction from data on Type Ia supernova via maximum likelihood estimation, using data from~\citet{davis07supernovae}. We focus on the fixed budget setting in this experiment as algorithms in the fixed budget settings are more practical and easier to adapt to real problems~\cite{karnin2013almost}. The likelihood is computed using the method from~\citet{shchigolev2017calculating}, and involves numerically evaluating integrals for each data point before a serial portion. We discretize the continuous parameter space into $64$ candidates (arms). We treat the likelihood of a parameter when computed using all $192$ points in the dataset as the mean of an arm. When sampling an arm, we compute the likelihood using $50$ points sampled with replacement from the dataset. Each arm pull can be allocated a number of processes to parallelize the set of numerical integrals that must be computed. First, we approximately compute the scaling function of the arm sampling as a function of its number of processes. In this particular scenario, the communication and synchronization costs are very limited, which results in fairly good scaling. Nevertheless, we plot the empirically estimated scaling function (Figure~\ref{fig:cosmological_scaling_plot}) and observe that the scaling function is close to $\lambda(m) = m^{-0.5}$. When computing $\kstar$, we approximately compute $\lambda^{-1}(t)$ by finding the smallest $m$ such that $\lambda(m) \leq t$. Additional experimental details for this experiment can be found in Section~\ref{app:cosmological_estimation}.
\begin{figure}
\centering     
\subfigure{\label{fig:cosmological_scaling_plot}\includegraphics[width=80mm]{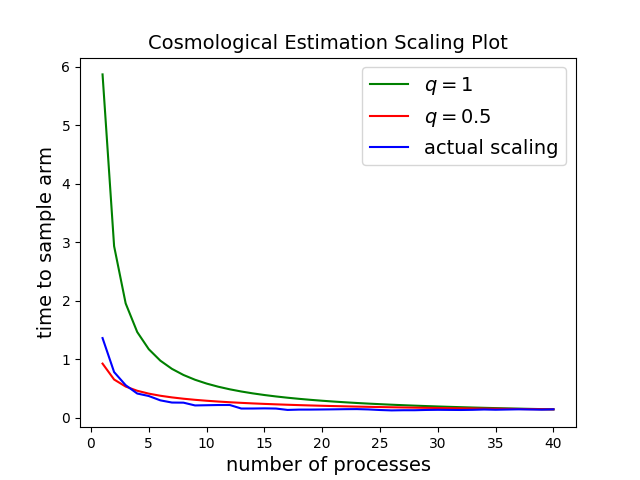}}
\subfigure{\label{fig:cosmological_success_resource_plot}\includegraphics[width=80mm]{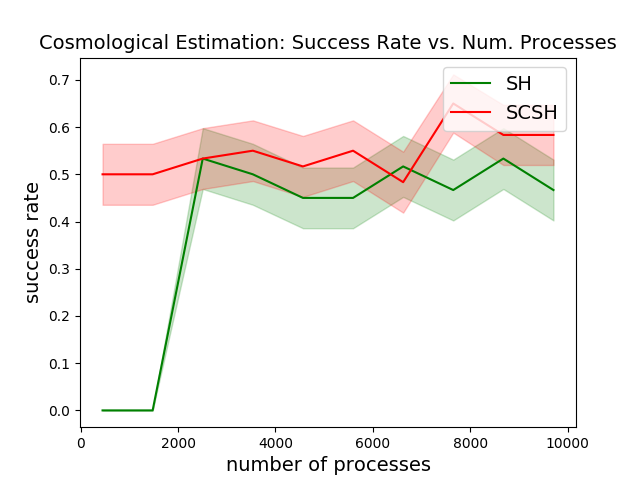}}
\caption{\small{\textbf{Left:} From a total pool of $40$ processes allowed per arm pull, we plot the empirical scaling function averaged over $10$ runs of the likelihood computation procedure and plot $\lambda(m) \propto m^{-q}$ for $q \in \{0.5, 1\}$ after calibrating $\lambda(1)$ (which corresponds to using $40$ processes). The behavior of the empirical scaling function is initially linear but becomes close to inverse square-root after $7$ processes. \textbf{Right:} We evaluate SSH on a cosmological parameter estimation task with $64$ arms, $T=0.56$, and vary the total number of processes available. We present the mean and standard error across $60$ runs. Although the scaling function is good ($q\approx 0.5$), SSH appears to consistently do as well or better than SH, especially when resources are very limited.}}
\end{figure}
In this experiment, we fix the time budget to be $0.56$ seconds. In this setting, $\kstar = 2$, so
each stage takes $0.28$ seconds, which requires at least $7$ processes per pull to complete pulls in
time. Sequential halving, which has $0.14$ seconds/stage, requires $25$ processes per pull to complete pulls in time. We observe that Algorithm~\ref{alg:fixed_budget} consistently matches or outperforms sequential halving in this setting (Figure~\ref{fig:cosmological_success_resource_plot}).

\section{Additional experimental details}
\label{app:experiments}

\subsection{Baseline hyperparameters}
\textbf{APR and BatchRacing:}
Empirically, we found that the deviation function used in the confidence intervals of
Algorithm~\ref{alg:fixed_confidence} and the baselines is very conservative, so we scale the
deviation by $0.2$ for all algorithms.
Scaling confidence intervals this way is a common technique in applied UCB-style bandit
work~\citep{kandasamy2015high,wang2017batched}.
Despite this scaling, all algorithms almost always identify
the best arm in the synthetic experiments. 

\textbf{UCB-E:} UCB-E~\cite{audibert2010best} requires a scale hyperparameter for the confidence
bounds. Specifically, the confidence bounds are computed as $(\hat{\mu} - \frac{a}{\sqrt{n}},
\hat{\mu} + \frac{a}{\sqrt{n}})$, where $\hat{\mu}$ is the empirical mean, $n$ is the number of times
the arm has been pulled, and $a$ is the scale parameter. We tune $a$ via grid search, and find that
when $q=0.9$, $a=1$ performs best empirically. For all other experiments, no value of $a$ helps, as
the time budgets are too small for a sequential algorithm to make sufficient progress.



\end{document}